\documentclass[10pt]{article}

\usepackage[utf8]{inputenc}
\usepackage[T1]{fontenc}

\usepackage{epsf}
\usepackage{amsmath}

\allowdisplaybreaks

\usepackage[showframe=false]{geometry}
\usepackage{changepage}

\usepackage{epsfig}
\usepackage{amssymb}

\usepackage{amsthm}
\usepackage{setspace}
\usepackage{cite}
\usepackage{mcite}

\usepackage{algorithmic}  
\usepackage{algorithm}

\usepackage{shadow}
\usepackage{fancybox}
\usepackage{fancyhdr}

\usepackage{color}
\usepackage[usenames,dvipsnames,svgnames,table]{xcolor}
\newcommand{\bl}[1]{\textcolor{blue}{#1}}
\newcommand{\red}[1]{\textcolor{red}{#1}}

\definecolor{mypurple}{rgb}{.4,.0,.5}

\usepackage[hyphens]{url}

\usepackage[colorlinks=true,
            linkcolor=black,
            urlcolor=blue,
            citecolor=purple]{hyperref}

\usepackage{breakurl}

\def\y{{\bf y}}
\def\v{{\bf v}}
\def\x{{\bf x}}

\def\x{{\mathbf x}}

\def\u{{\bf u}}
\def\v{{\bf v}}
\def\x{{\bf x}}
\def\y{{\bf y}}

\def\q{{\bf q}}

\def\c{{\bf c}}
\def\d{{\bf d}}

\def\h{{\bf h}}

\def\cC{{\mathcal C}}
\def\cB{{\mathcal B}}
\def\cA{{\mathcal A}}
\def\cH{{\mathcal H}}

\def\be{\begin{equation}}
\def\ee{\end{equation}}
\def\ba{\left[\begin{array}}
\def\ea{\end{array}\right]}

\def\u{{\bf u}}
\def\v{{\bf v}}
\def\x{{\bf x}}
\def\y{{\bf y}}

\def\q{{\bf q}}

\def\c{{\bf c}}
\def\d{{\bf d}}

\def\p{{\bf p}}

\def\1{{\bf 1}}

\def\0{{\bf 0}}

\def\erf{\mbox{erf}}
\def\erfc{\mbox{erfc}}

\def\calX{{\cal X}}
\def\calY{{\cal Y}}

\def\mR{{\mathbb R}}
\def\mN{{\mathbb N}}
\def\mE{{\mathbb E}}
\def\mS{{\mathbb S}}
\def\mP{{\mathbb P}}

\def\lp{\left (}
\def\rp{\right )}

\def\y{{\bf y}}
\def\v{{\bf v}}
\def\x{{\bf x}}

\def\x{{\mathbf x}}

\def\u{{\bf u}}
\def\v{{\bf v}}
\def\x{{\bf x}}
\def\y{{\bf y}}

\def\q{{\bf q}}

\def\c{{\bf c}}
\def\d{{\bf d}}

\def\h{{\bf h}}

\def\cB{{\cal B}}
\def\cH{{\cal H}}

\def\be{\begin{equation}}
\def\ee{\end{equation}}
\def\ba{\left[\begin{array}}
\def\ea{\end{array}\right]}

\def\u{{\bf u}}
\def\v{{\bf v}}
\def\x{{\bf x}}
\def\y{{\bf y}}

\def\q{{\bf q}}

\def\c{{\bf c}}
\def\d{{\bf d}}

\def\p{{\bf p}}

\def\({\left (}
\def\){\right )}

\def\1{{\bf 1}}

\def\q{{\bf q}}

\def\0{{\bf 0}}

\def\cX{{\mathcal X}}
\def\cY{{\mathcal Y}}

\usepackage{xcolor}
\usepackage{color}

\definecolor{darkgreen}{rgb}{0, 0.4,0}

\definecolor{purplebrown}{rgb}{0.5,0.1,0.6}

\definecolor{ultclupcol}{rgb}{0.1,0.5,0.5}

\definecolor{mytrycolor}{rgb}{0.5,0.7,0.2}

\definecolor{ultclupcola}{rgb}{.5,0,.5}

\definecolor{shadebrown}{rgb}{0.1,0.1,0.9}
\definecolor{lightblue}{rgb}{0.2,0,1}

\usepackage{fancybox}
\usepackage{graphicx}
\usepackage{epstopdf}
\usepackage{epsfig}
\usepackage{wrapfig}
\usepackage{subfigure}

\usepackage{xcolor}
\usepackage{tcolorbox}
\tcbuselibrary{skins}

\newtcbox{\xmybox}{on line,
arc=7pt,
before upper={\rule[-3pt]{0pt}{10pt}},boxrule=0pt,
boxsep=0pt,left=6pt,right=6pt,top=0pt,bottom=0pt,enhanced, coltext=blue, colback=white!10!yellow}

\newtcbox{\xmyboxa}{on line,
arc=7pt,
before upper={\rule[-3pt]{0pt}{10pt}},boxrule=0pt,
boxsep=0pt,left=6pt,right=6pt,top=0pt,bottom=0pt,enhanced, colback=white!10!yellow}

\newtcbox{\xmyboxb}{on line,
arc=7pt,
before upper={\rule[-3pt]{0pt}{10pt}},boxrule=1pt,colframe=darkgreen!100!blue,
boxsep=0pt,left=6pt,right=6pt,top=0pt,bottom=0pt,enhanced, colback=white!10!yellow}

\newtcbox{\xmyboxc}{on line,
arc=7pt,
before upper={\rule[-3pt]{0pt}{10pt}},boxrule=.7pt,colframe=blue!100!blue,
boxsep=0pt,left=6pt,right=6pt,top=0pt,bottom=0pt,enhanced, coltext=blue, colback=white!10!yellow}

\newtcbox{\xmytboxa}{on line,
arc=7pt,
before upper={\rule[-3pt]{0pt}{10pt}},boxrule=.0pt,colframe=pink!50!yellow,
boxsep=0pt,left=6pt,right=6pt,top=0pt,bottom=0pt,enhanced, coltext=white, colback=blue!40!red}

\newtcbox{\xmytboxb}{on line,
arc=7pt,
before upper={\rule[-3pt]{0pt}{10pt}},boxrule=.0pt,colframe=pink!50!yellow,
boxsep=0pt,left=6pt,right=6pt,top=0pt,bottom=0pt,enhanced, coltext=white, colback=white!40!green}

\makeatletter
\newcommand\subsubsubsection{\@startsection{paragraph}{4}{\z@}{-2.5ex\@plus -1ex \@minus -.25ex}{1.25ex \@plus .25ex}{\normalfont\normalsize\bfseries}}
\newcommand\subsubsubsubsection{\@startsection{subparagraph}{5}{\z@}{-2.5ex\@plus -1ex \@minus -.25ex}{1.25ex \@plus .25ex}{\normalfont\normalsize\bfseries}}
\makeatother

\newtheorem{theorem}{Theorem}

\newtheorem{corollary}{Corollary}

\begin{document}

\begin{singlespace}

\title {Theoretical limits of descending $\ell_0$ sparse-regression ML algorithms  
}
\author{
\textsc{Mihailo Stojnic
\footnote{e-mail: {\tt flatoyer@gmail.com}} }}
\date{}
\maketitle

\centerline{{\bf Abstract}} \vspace*{0.1in}

We study the theoretical limits of the $\ell_0$ (quasi) norm based optimization algorithms when employed for solving classical compressed sensing or sparse regression problems. Considering standard contexts with deterministic signals and statistical systems, we utilize \emph{Fully lifted random duality theory} (Fl RDT) and develop a generic analytical program for studying performance of the \emph{maximum-likelihood} (ML) decoding.  The key ML performance parameter, the residual \emph{root mean square error} ($\textbf{RMSE}$), is uncovered to exhibit the so-called \emph{phase-transition} (PT) phenomenon. The associated aPT curve, which separates the regions of systems dimensions where \emph{an} $\ell_0$ based algorithm succeeds or fails in achieving small (comparable to the noise) ML optimal $\textbf{RMSE}$ is  precisely determined as well. In parallel, we uncover the existence of another dPT curve which does the same separation but for practically feasible \emph{descending} $\ell_0$ ($d\ell_0$) algorithms. Concrete implementation and practical relevance of the Fl RDT typically rely on the ability to conduct a sizeable set of the underlying numerical evaluations which reveal that for the ML decoding the Fl RDT converges astonishingly fast with corrections in the estimated quantities not exceeding $\sim 0.1\%$ already on the third level of lifting. Analytical results are supplemented by a sizeable set of numerical experiments where we implement a simple variant of $d\ell_0$ and demonstrate that its practical performance very accurately matches the theoretical predictions. Completely surprisingly, a remarkably precise agreement between the simulations and the theory is observed for fairly small dimensions of the order of 100.

\vspace*{0.25in} \noindent {\bf Index Terms: Compressed sensing; Sparse regression; $\ell_0$ norm; ML decoding; Random duality}.

\end{singlespace}

\section{Introduction}
\label{sec:back}

Studying structured objects has been a topic of strong applied and theoretical research interest for several decades. It is, however, the appearance of seminal works  \cite{Donoho06CS,CRT,CanTao07,CT} by Donoho, Candes, Romberg, and Tao that played a key role in transforming these studies and bringing them under the umbrella of immensely popular field called compressed sensing (CS). A direct relation between CS and statistical sparse regression (SR) notion widened the horizon of applicability contributing to a further rapid evolution with a strong influence on and interconnection with other scientific fields. Today compressed sensing is undoubtedly among the most recognizable scientific brands with applications spanning signal and image processing, vision and machine learning, health and social sciences, and beyond. Anticipated surge in the demand for efficient data processing is expected to only strengthen and expand the compressed sensing role as we transition into the big data era in the coming years.

Given its elite scientific positioning, it is a no surprise that CS has been associated with avalanches of top class research over the last two decades. Of particular importance is a multi-faceted CS relevance ranging from the real-world applications and engineering implementations, to practical algorithmic designs, accompanying theoretical analyses, and more. This paper focuses on theoretical aspects and studies both state of the art CS limits and their computationally achievable counterparts. In particular, we consider noisy systems and the associated \emph{maximum-likelihood} (ML) decoding. Due to the discrete nature of sparsity -- a key structural property accompanying  both CS and SR -- the ML decoding has been known for a long time as principally realizable but practically computationally likely infeasible concept. In a more mathematical terminology, as an NP problem (see, e.g., \cite{Natarajan95}), it has been known to be solvable but not necessarily with polynomial algorithms. Such a generic notion was challenged through the appearance of convex heuristics. The so-called $\ell_1$ minimization could solve the noiseless ML decoding counterpart for a range of system parameters in a diverse set of randomized instances. While its corresponding noisy LASSO analogue could not fully solve the ML, it still worked fairly well, thereby strengthening $\ell_1$'s practical relevance. These developments strongly suggested that the NP notions based limitations, predicated on the worst case algorithmic behavior, might be practically circumventable. In parallel with the appearance of convex techniques, their phase-transitioning (PT) behavior emerged as well. Such a behavior was reflected through a sharp transition from the system parameters regions where the heuristics succeed to those where they fail. Demarcating these regions and the associated separating line -- the so-called PT curve -- effectively portrayed a computational complexity picture with the following two key takeaways:

\begin{itemize}
  \item The NP-ness is \emph{partially beatable} (there are system parameters regions where problems are generally
NP but fast heuristics typically succeed in solving them).
  \item A residual \emph{computational gap} (C-gap) that prevents completely practically overcoming NP-ness still remains (there are system parameters  regions where the fast heuristics fail).
\end{itemize}

Such a pictorial understanding, although widely accepted, remained anecdotal for years, until rigorous mathematical confirmations began to emerge. Three lines of work played a key role: 1) Donoho's integral geometry approach \cite{DonohoPol,DonohoUnsigned}; 2) Donoho, Maleki, Montanari, and Bayati message passing approach \cite{DonMalMon09,BayMon10}; and 3) Stojnic's probabilistic optimization approach \cite{StojnicCSetam09,StojnicICASSP10var,StojnicRegRndDlt10}. They addressed a CS central mathematical challenge and pinpointed the precise location of $\ell_1$'s PT. On the other hand, their extensions \cite{DonohoMM11,BayMon10lasso,StojnicGenLasso10,StojnicGenSocp10} have done the same for the noisy LASSO counterparts. Many further considerations followed addressing similar challenges in different relevant CS scenarios with various other forms of structuring (see, e.g. \cite{DonohoSigned,DonohoMM11,BayMon10lasso,StojnicICASSP10block,StojnicISIT2010binary,OH10}). The main conclusion was precisely as stated earlier -- the convex relaxations can circumvent the NP-ness but only to a degree. Moreover, while this answered questions regarding remaining convex heuristics C-gaps, it in now way could rule out potential ability of  different heuristics to do better.

Ensuing works over the better part of the last two decades focused on further bridging the C-gaps. A lot of progress has been made mostly within the so-called \emph{Bayesian} context where the prior knowledge of the signal is available and can be utilized in the recovery process (see, e.g., \cite{ReevesP19,ReevesG12,ReevesG13,ReevPfis2016,BarbierKMMZ18,KMSSZ12,KMSSZ12a,WuV10a,WuV11,WuV12a,WuV12b}). A similar notion of the C-gap (this time relating to presumably polynomial algorithms) is formulated and precisely characterized. Associated practical implementations usually rely on the approximate message passing (AMP) (as in \cite{DonMalMon09,BayMon10,DonohoMM11,BayMon10lasso}) or similar belief/survey propagation related techniques. Despite a great progress, the key challenges still remain: \emph{\textbf{(1)}} determining  the C-gap in the ML context; and \emph{\textbf{(2)}} bridging the C-gap. The first one implicitly assumes establishing the ML's ultimate theoretical limits, whereas the second one focuses on the algorithmic developments to approach such limits. Our study takes significant steps forward in both these directions. We introduce a novel approach to study underlying phenomena and precisely characterize the ultimate and computationally achievable ML decoding phase transitions. Moreover, it allows us to implement a simple algorithmic procedure which matches the theoretical predictions with a remarkable accuracy even in systems with dimensions of the order of 100.

\section{Compressed sensing -- sparse regression setup}
 \label{sec:mathsetup}

We assume the following standard noisy compressed sensing linear setup
\begin{eqnarray}
\y=A\bar{\x}+\sigma\v, \label{eq:ex1}
\end{eqnarray}
where $\bar{\x}\in\mR^n$ is a unit norm $k$-sparse vector, $A\in\mR^{m\times n}$ is a system matrix, $\v\in\mR^m$ is a noise vector, and $\sigma\in\mR$ is a scaling factor (due to overall scaling invariance and presence of $\sigma$, unit norm can be imposed on $\bar{\x}$ without loss of generality). Throughout the presentation, we consider high-dimensional scenarios and in particular their analytically hardest, so-called \emph{linear} (sometimes also referred to as \emph{proportional}), regimes with
 \begin{eqnarray}
\alpha  \triangleq    \lim_{n\rightarrow \infty} \frac{m}{n},  \qquad \beta  \triangleq    \lim_{n\rightarrow \infty} \frac{k}{n}, \label{eq:ex15}
\end{eqnarray}
and $\alpha\in[0,1]$ and $\beta\in[0,1]$ being the \emph{under-sampling} and \emph{sparsity} ratios, respectively.

Linear system (\ref{eq:ex1}) has a centuries long history and appears as an unavoidable fundamental modeling building block in a host of disciplines, including among others transmission, compression, estimation/prediction, classification, and inference systems. Accordingly, the terminology associated with its key pieces varies with systems matrix, for example, often being referred to as the channel, coding/decoding, compression, dictionary or data  processing matrix. As the results that we present below are generic and applicable to any field where the models can be utilized, we will interchangeably use all of the known terminologies for any  of the model's components. Along the same lines, we will often interchangeably refer to (\ref{eq:ex1}) as the basic compressed sensing (CS) or sparse regression (SR) model.

We distinguish two parts of the model: \textbf{\emph{(i)}} the first one that relates to the infrastructure and consists of $A$ and $\v$  and the way they interact/operate together (in case of (\ref{eq:ex1}) it is a linear affine interaction); and \textbf{\emph{(ii)}} the second one that relates to input/outut data and is represented by $\bar{\x}$ and $\y$. As mentioned above, depending on the context, the first one may sometimes be called the transmission, compression, coding or acquisition system. Analogously, the second one may be referred to as unprocessed/processed or sent/received  signal, codeword, or data set.
In almost all contexts where (\ref{eq:ex1}) is applicable, the key aspect of its utilization usually boils down to the following principle: a (presumably) structured vector $\bar{\x}$ is processed through (\ref{eq:ex1}) and as a result its an image $\y$ is obtained. Relying on the knowledge of $A$ and a lack of knowledge of $\v$ the goal is to recover $\bar\x$. In other words, the goal is to determine
\begin{eqnarray}
\hat{\x}=f_{rec}(\y,A), \label{eq:ex2}
\end{eqnarray}
where $f_{rec}(\cdot)$ is a function of choice that typically depends on the practical scenario where the model originates from.

To facilitate the exposition, throughout the paper we assume a statistical scenario where the components of $\v$ are iid standard normal. In such a scenario one has for the pdf of $\v$
\begin{eqnarray}
p(\v)\sim e^{-\frac{\|\v\|_2^2}{2}}, \label{eq:ex3}
\end{eqnarray}
and consequently for a given fixed $A$
\begin{eqnarray}
p(\y|\x,A)\sim e^{-\frac{\|\y-A\x\|_2^2}{2\sigma^2}}. \label{eq:ex4}
\end{eqnarray}
The so-called \emph{maximum-likelihood} (ML) criterion then suggests the following way for estimating $\hat{\x}$ in (\ref{eq:ex2})
\begin{eqnarray}
\hspace{-1.4in}\bl{\text{\textbf{\emph{ML decoding:}}}} \qquad \qquad\hat{\x}=\mbox{arg} \max_{\x} p(\y|\x) = \mbox{arg}\min_{\x} \|\y-A\x\|_2, \label{eq:ex5}
\end{eqnarray}
which ultimately implies
\begin{eqnarray}
f_{rec}(\y,A)= \mbox{arg}\min_{\x} \|\y-A\x\|_2, \label{eq:ex6}
\end{eqnarray}
as a formal definition of the above mentioned ML recovery function. We here consider a structured recovery with $\bar{\x}$ being $k$-sparse (under $k$-sparse we assume a vector that has no more than $k$ nonzero components). Moreover, to amplify the logical fairness (particularly of the theoretically limiting part) of the recovery process, we assume a prior awareness of $\bar{\x}$'s sparse structuring. Availability of such a knowledge can then be used  to (potentially) strengthen (\ref{eq:ex5}) through  the following
\begin{eqnarray}
\hspace{-1.4in}\bl{\ell_0  \quad  \text{\textbf{\emph{ML decoding:}}}} \qquad \qquad \hat{\x}= \mbox{arg}\min_{\x,\|\x\|_0=k} \|\y-A\x\|_2, \label{eq:ex7}
\end{eqnarray}
where, as usual, $\|\cdot\|_0$ is referred to as the (quasi) $\ell_0$ norm and it counts the number of its argument's nonzero elements (to make writing shorter, throughout the presentation we may call (\ref{eq:ex7}) simply just ML decoding without emphasizing its associated $\ell_0$ property). Of our prevalent interest will be -- what is typically viewed as both analytically and algorithmically the hardest -- scenario where $\bar{\x}$ is a fixed \emph{deterministic} a priori unknown $k$-sparse vector. We emphasize that this is very different from the scenarios where $\bar{\x}$ is of statistical nature or from the scenarios where either statistical or deterministic nature of $\bar{\x}$ is known and available in the recovery process. Two key differences should be noted: \textbf{\emph{(i)}} the considered scenario is generally viewed as more realistic since it allows for many practical situations where gathering prior information about $\bar{\x}$ is infeasible; and \textbf{\emph{(ii)}} it is expected to be analytically much harder as one has no reason to, a priori, believe that simplifying analytical features akin to replica symmetric or one step of replica symmetry breaking behavior will appear.

\subsection{Differences/similarities with existing recovery methods}
 \label{sec:simdiff}

The $\ell_0$ ML decoding strategy from (\ref{eq:ex7}) is the best type of recovery process that one can design within the ML context. In the noiseless case ($\sigma\rightarrow 0$), it basically corresponds to the so-called $\ell_0$ recovery which always exactly recovers $\bar{\x}$ (throughput the paper we assume $\bar{\x}$'s uniqueness, i.e., we assume that there is no $k$-sparse $\x$ such that $A(\bar{\x}-\x)=0$ and $\bar{\x}\neq\x$). The famous convex relaxation methods are conceptually similar to (\ref{eq:ex7}), with the constraints typically imposed on $\ell_1$ instead of $\ell_0$ norm. As $\ell_1$ is the tightest (convex) norm relaxation of $\ell_0$ it is predicated to be the best that convex methods can do. In the noiseless case one of the most prominent features of the $\ell_1$ relaxation is the appearance of the  phase transition (PT) phenomenon. The associated PT curve delimits the system's parameters (sparsity and under-sampling) regions depending on whether the so-called $\ell_0-\ell_1$ equivalence happens or not. In the noisy case, $\ell_1$ relaxation based methods (as well as any other ones) can not recover $\bar{\x}$ which implies the appearance of the nonzero residual errors. These errors themselves exhibit phase transitions where they typically change from a value proportional to the noise (i.e. $\sigma$) to a much larger (typically an order of magnitude larger) value. 
Both the phase transitions and the residual errors of the convex methods typically lag behind the theoretically best possible ones by rather healthy margins (see, e.g., \cite{DonohoPol,StojnicCSetam09,StojnicICASSP10var,StojnicUpper10} for the noiseless PTs and e.g. \cite{DonohoMM11,BayMon10,StojnicGenLasso10,StojnicGenSocp10,WuV11,WuV12a,WuV12b} for the noisy errors and their associated PTs). Existence of these gaps coupled with computationally not so favorable large scale convex relaxations implementations underlined some of the most critical CS/SR theoretical and algorithmic shortcomings.

Consequently, as mentioned earlier, a strong wave of research  aimed at remedying these shortcomings ensued over the last a couple of  decades. Several lines of work relying on utilization of non-convex approximate message passing (AMP) or general belief/survey propagation algorithmic concepts turned out to be very fruitful (see, e.g, \cite{DonMalMon09,DonohoMM11,DonohoJM13a,DonohoJM13,KMSSZ12,KMSSZ12a,BarbierKMMZ18,RanganFG12,RanganFG10,RanganFG09,SchniterRF16,RanganSF17,BarbierLSKZ23,BarbierLSKZ24} and references therein). However, almost all of them rely on a statistical nature of $\bar{\x}$ which is not present in our setup. Statistical \emph{identicalness} across all $n$ components of $\bar{\x}$ is typically featured as well (see, e.g., \cite{DonohoJM13a,DonohoJM13}) and, when $\bar{\x}$'s prior is available, often coupled with the utilization of the \emph{Bayesian} inference or \emph{maximum a posteriori} (MAP) concepts (see, e.g. \cite{KMSSZ12,KMSSZ12a,BarbierKMMZ18,ReevesG12,ReevesG13,ReevPfis2016,WuV12b,WuV12a,RanganFG12,RanganFG10,RanganFG09,BereyhiMS17,BereyhiMS19} for CS related considerations as well as, e..g., \cite{GuoV05,Tanaka02,GuoSV05,MerhavGS10,GuoWV11,WuV11,WuV10a} for earlier CDMA and mutual information-MMSE related ones) allowing to reach a more manageable objective. In particular, in the context of Bayesian inference one focuses on sampling the \emph{posterior} distribution
\begin{eqnarray}
p(\x|\y)\sim p(\x)p(\y|\x), \label{eq:exbayes1}
\end{eqnarray}
which for a given fixed $A$ becomes
\begin{eqnarray}
p(\x|\y)\sim p(\x)p(\y|\x) \sim p(\x)e^{-\frac{\|\y-A\x\|_2^2}{2\sigma^2}}. \label{eq:exbayes2}
\end{eqnarray}
Clearly, coupling with $p(\x)$ on the right hand side makes it substantially different from what is given in (\ref{eq:ex7}). Moreover, in addition to knowing $p(\x)$,  one needs a few extra things to fall in place to successfully utilize (\ref{eq:exbayes2}). In particular, the system's statistics need to be correctly postulated and on top of that $\sigma$ must be known as well (the normalization also needs to be accounted for as it is now nontrivial). Once these assumptions are in place, the above transformation turns out to be particularly convenient from the analytical point of view. Namely, after framing it into the corresponding statistical mechanics context one ends up facing the so-called high (or moderately high) temperature regime where
the replica symmetric or the first level of replica symmetry breaking behavior usually happens making the underlying analyses manageable. While the residual analytical considerations are still very challenging (see, e.g. \cite{KMSSZ12,KMSSZ12a,BarbierKMMZ18,ReevPfis2016,WuV12b,WuV12a}), they are substantially easier than the ones faced in the so-called jointly optimal ML decoding. It should also be noted that MAP estimations can be positioned so that they correspond to a regularized version of (\ref{eq:ex7}) (see, e.g., \cite{BereyhiMS17,BereyhiMS19}). However, the luxury of a priori believing in the replica symmetric or one step of replica symmetry breaking behavior quickly fades away. Some methods also allow for an additional structuring of $A$ (see, e.g., \cite{DonohoJM13a,KMSSZ12,KMSSZ12a,FHicassp,Tarokh,MaVe05,DTbern,DonTan09Univ}). None of the so to say helping concepts (statistically identical $\bar{\x}_i$s, Bayesian inference, additionally structured $A$ etc.) are possible within the context that we consider. At the same time, it is important to note that pretty much all of the mentioned lines of work provide various improvements over the classical convex relaxations -- noiseless $\ell_1$ or the corresponding noisy LASSO/SOCP analogues. These are typically manifested through a substantially lowered computational complexity, more favorable phase transitions (allowing for larger sparsity/under-sampling ratios), or significantly better noise robustness guarantees. Nonetheless, beating convexity concepts ($\ell_1$ or LASSO) in a typical plain scenario (unchangeable $A$, deterministic $\bar{\x}$, and/or not so large dimensions) remains a formidable challenge. With the appearance of Controlled loosening-up (CLuP) algorithms it became clear that convexity methods (in particular LASSO/SOCP and their associated derivatives) are universally beatable in the regimes well below phase transitions (see, e.g., \cite{Stojnicclupspreg20}). Interestingly, results that we present below, suggest that the phase transitions themselves can be beat as well.

\subsection{Contextualization within the key prior work milestones}
\label{sec:prior}

The above results are among the most relevant ones directly related to our work. We find it useful to recall on some general CS/SR milestones as well. These will help properly understand the position and importance of the preceding discussion in the overall CS/SR mosaic. We separate four streams of closely related prior work and focus on their most representative results. As a substantial majority of these results is obtained in statistical contexts, we often skip reemphasizing it and instead adopt the terminology convention of implicitly assuming that they hold with probability going to 1 as $n\rightarrow\infty$ (the very same convention is often utilized in the presentation of our own results later on as well).

\emph{\textbf{(i)} \underline{Qualitative (non phase-transitional) characterizations:}} Early compressed sensing developments put a strong emphasis on efficient implementations and fast recovery algorithms. The so-called orthogonal matching pursuit (OMP) turned out to be a particulary successful one. As shown in, e.g., \cite{JATGomp,JAT,NeVe07}, assuming $\sigma=0$,
OMP (or a slightly modified OMP) with a running complexity $O(n^2)$ could recover $k$-sparse $\x$ in (\ref{eq:ex1})
 provided that $m=O(k\log(n))$ (a stage-wise
OMP from \cite{DTDSomp} could even drop the complexity to $O(n \log n)$). This analytically reaffirmed strong ``\emph{sublinear barrier}'' as a major theoretical CS challenge. At the same time, there was a strong numerical evidence dating back to 70s and 80s of the last century that convex relaxation based heuristics (slightly slower than OMP and typically synonymized with $\ell_1$) could do better when it comes to the recoverable sparsity. In fact, as the interest in these techniques picked up in the 90s, it became evident that they have a strong potential to break the sublinear barrier. It wasn't until 2004, however, when the breakthrough papers of Donoho, Candes, Romberg, and Tao \cite{Donoho06CS,CRT,CanTao07,CT} unveiled a collection of mathematical results that rigorously confirmed the existence of the $\ell_0-\ell_1$ equivalence phenomenon (a property that $\x$ with minimal $\ell_1$ norm is the sparsest solution of noiseless (\ref{eq:ex6})). Moreover, they confirmed that for any $\alpha>0$ there is a $\beta>0$ such that $\ell_0-\ell_1$ equivalence happens, thereby also confirming
the $\ell_1$'s  \emph{linear} sparsity recoverability  property -- a key sublinear barrier breaking analytical result and undoubtedly one of the monumental milestones in the development of compressed sensing. Later on, the progress continued and even OMP like improvements CoSAMP (see e.g. \cite{NT08}) and Subspace pursuit (see e.g. \cite{DaiMil08}) were developed with the ability to attain linear recoverability in polynomial time. Many other excellent developments followed as well addressing some of key CS challenges and further raising awareness about CS and contributing to its a gigantic overall importance and popularity (see, e.g., \cite{KWT09,KabashimaWT10,Wainwright09a,Wainwright09,SPH,RFPrank,CR09matcomp,CT10matcomp,KMO10matcomp,Klopp14matcomp,KLT11matcomp,NW11matcomp,NW12matcomp,RT11matcomp} and references therein for thorough discussions regarding a host of theoretical and algorithmic aspects ranging from signifying the importance of various performance metrics to handling different forms of recovery robustness and alternative types of algorithms and signal structuring).

\emph{\textbf{(ii)} \underline{Quantitative phase-transitional characterizations -- noiseless $\ell_1$:}}  While the discovery  of $\ell_1$'s \emph{linear} sparsity recoverability represented one of the key compressed sensing milestones, it was a \emph{qualitative} analytical characterization. It affirmed that $\ell_1$ is a good heuristic, but did not precisely specify how good it actually is. This soon motivated a  search for analytically more precise (\emph{quantitative}) $\ell_1$ performance characterizations. Utilizing high-dimensional geometry tools and studying associated cross-polytopal neighborliness, Donoho in \cite{DonohoPol,DonohoUnsigned} \emph{precisely} identified values of $\beta$ for any $\alpha>0$ (together with Tanner,  in \cite{DonohoSigned,DT}, they extended characterization from the cross-polytopal to simplicial neighborliness  and consequently widened the range of applicability to the so-called nonnegative signals). Their research analytically revealed the phase transition (PT) phenomenon in $\ell_1$
which implied the existence of a distinct boundary (PT curve) between the regions of system parameters where the heuristic, specifically the $\ell_1$ (or the $\ell_1^+$ for the nonnegative signals), is successful or not. Stojnic's later Random Duality Theory (RDT) \cite{StojnicCSetam09,StojnicICASSP10var,StojnicRegRndDlt10} based contributions broadened the scope, addressing various optimization problems and random structures. These results established that the $\ell_1$ and $\ell_1^+$ phase transitions could be represented as a single line functional equations. Very recently, \cite{CapSto24,CapStoptrelISIT23} uncovered that  $\ell_1$ and $\ell_1^+$ PTs are further connected through remarkably simple explicit relations. Another important relevant line of work is a hybrid PT geometrical approach utilized in \cite{ALMT14}. It uses the same geometrical results as Donoho but in a slightly different formulation that allows for more convenient analytical considerations via Gaussian widths. The widths are eventually handled following Stojnic's RDT approach developed in \cite{StojnicCSetam09,StojnicISIT2010binary,StojnicICASSP10block,StojnicICASSP10var}. Finally, one should also add that in a series of papers that provided ultimate probabilistic settling of $\ell_1$ compressed sensing (see, e.g., \cite{Stojnicl1RegPosfinn} and references therein), Stojnic developed his own way of handling underlying polytopal angle geometry and was  able to move things way beyond the phase transitions. In particular, he actually \emph{precisely} determined the \emph{entire} transitional distribution for \emph{any} (including all \emph{finite} ones as well) set of problem dimensions. Clearly, a result that can not be further moved or extended in any form or shape and represents the ultimate hope of any analytical characterization approach.

\emph{\textbf{(iii)} \underline{Quantitative phase-transitional characterizations -- noisy $\ell_1$:}} The outlined results unequivocally defined the $\ell_1$ PTs, addressing a major mathematical CS conundrum. Subsequent research expanded on these foundational findings, exploring various facets of PT phenomena (refer to \cite{DonohoMM11,BayMon10lasso,StojnicGenLasso10,StojnicGenSocp10,StojnicICASSP10block,StojnicISIT2010binary,OH10} for examples). Given that the ML decoding is the primary topic of this paper and that it is by the definition associated with the noisy systems, of particular interest are the noisy counterparts to the above noiseless ones. Their development turned out to be a bit more challenging. To improve on not overly impressive $\ell_1$ large scaling capabilities, Donoho, Maleki, and Montanari introduced Approximate Message Passing (AMP) methodology in \cite{DonMalMon09}, presenting an impressively fast alternative to $\ell_1$. Using a state evolution analysis they predicted a perfect alignment between $\ell_1$'s and AMP's PTs. This alignment was subsequently rigorously validated by Bayati and Montanari in \cite{BayMon10}. \cite{DonohoMM11} extended the AMP to noisy setups. Relying again on a state evolution analysis, they precisely determined the AMP's worst-case $\textbf{RMSE}$s and uncovered that they also exhibit phase-transition phenomena. Bayati and Montanari in \cite{BayMon10lasso} rigorously confirmed these predictions and established an equivalence between the noisy AMPs and LASSOs. In a separate line of work, Stojnic utilized RDT and extended his noiseless $\ell_1$ PT analyses to the noisy LASSO counterparts  (see, e.g., \cite{StojnicGenLasso10}). Moreover, he established that LASSO's SOCP analogues achieve exactly the same results, thereby establishing a LASSO-SOCP equivalence as well (see, e.g., \cite{StojnicGenLasso10,StojnicGenSocp10}).

\emph{\textbf{(iv)} \underline{Non $\ell_1$ phase-transitional characterizations:}} The above AMP considerations turned out to have equivalent convexity based heuristics counterparts.  As such they could improve on the large scale computational aspects, but the theoretical gaps, characteristic for convexity methods, remained. That motivated studying AMPs (or other non-convex) methods  without making connections to the convexity ones and the avalanche of the above mentioned results \cite{KMSSZ12,KMSSZ12a,BarbierKMMZ18,ReevesG12,ReevesG13,ReevPfis2016,WuV12b,WuV12a,WuV11,DonohoJM13a,WengZMW16,ZhengMWWL17,BoraJPD17,DonohoJMM11} appeared (for more on Controlled loosening-up (CLuP) non-AMP alternatives see, e.g., \cite{Stojnicclupint19,Stojnicclupspreg20}). The conclusions were mixed. For example, Bayesian work (see, e.g.,  \cite{KMSSZ12,KMSSZ12a,BarbierKMMZ18,ReevesG12,ReevesG13,ReevPfis2016,WuV12b,WuV12a,WuV11}) allowed improvements of the associated  AMPs over the plain $\ell_1$s at the expense of the assumed prior signal's knowledge. Existence of C-gaps, akin to those observed for the convexity methods, was discovered as well essentially implying that bridging such gaps might not be possible while solely relying on the signals' priors. On the other hand, additional structuring of $A$, pursued in, e.g.,  \cite{DonohoJM13a,KMSSZ12,KMSSZ12a,FHicassp,Tarokh,MaVe05,DTbern,DonTan09Univ}, turned out to be sufficiently  helpful in that direction. The so-called spatial coupling techniques, first considered via the replica methods in \cite{KMSSZ12,KMSSZ12a} and later rigorously confirmed in \cite{DonohoJM13a}, showed that AMP like algorithms can completely bridge the C-gap if a proper structuring of $A$ is allowed and undertaken. Earlier,   \cite{FHicassp,Tarokh,MaVe05} showed that Reed-Solomon based coding/decoding strategies could be employed for the design of $A$ and decoding of $\bar{\x}$ so that a similar level of success is achieved in the so-called \emph{strong} PT sense. Related specific cyclic-polytope type of structuring of $A$ was further tailored for the corresponding success in the recovery of the nonnegative signals (see, e.g., \cite{DTbern,DonTan09Univ}). It should also be noted that while the above mentioned Byesian methods can improve on the plain $\ell_1$, their associated AMPs can not achieve such an improvement for every type of signals prior. A very successful attempt in that direction was provided in \cite{DonohoJM13a} while connecting to minmax denoising and utilizing various forms of the so-called AMP shrinkage operators. Even though  the improvement was not nominally large, it showed in principle  that the PTs of convex based methods are likely universally beatable.

As discussed in the previous section, despite an excellent progress that all of the above lines of works have made, one needs to keep in mind that none of the assumed concepts (beyond the plain $\ell_1$) that almost all of them utilize are allowed within the context considered in this paper.

\subsection{Our contributions}
\label{sec:contrib}

The above discussed very strong progress in understanding various CS aspects ultimately emphasizes that $\ell_1$ is generically still very difficult to beat by fast algorithms. Moreover, the limits of the theoretically best (but not necessarily fast), $\ell_0$ ML decoding, remained unknown so far. Recognizing these two major CS challenges, our work aims to provide a strong progress on both fronts. Here is a brief summary of our main results obtained in a statistical large dimensional context assuming the so-called \emph{linear/proportional} regime with a \emph{fixed deterministic a priori unknown} $\bar{\x}$, no  allowed structural changes of $A$, and a small noise:

\begin{itemize}
  \item Utilizing Fully lifted random duality theory (Fl RDT), we introduce an innovative method to design a generic program for studying both theoretical and algorithmic aspects of $\ell_0$ ML decoding (see Section \ref{sec:prac}).
 \item  We provide a precise characterization of two key  $\ell_0$ ML decoding  quantities, the objective value and the (noise scaled) residual root mean squared error ($\textbf{RMSE}$) (see Section \ref{sec:nuemricalags} and Theorem \ref{thm:negthmprac1}).
  \item A crucially important \emph{non-monotonic} behavior of the objective, when viewed as a function of the overlap between the true and the estimated signal, is observed (see, Figures \ref{fig:fig1}-\ref{fig:fig6}). The system dimensions relations where the non-monotonicity is further coupled with the appearance of local optima are uncovered, precisely characterized, and directly connected to the existence of multiple phase-transition (PT) phenomena (see, Figures \ref{fig:fig1}-\ref{fig:fig6}). Consequently, the associated PT curves are precisely determined as well (see Figure \ref{fig:fig7}).
  \item In particular, we uncover the location of aPT curve which splits $(\alpha,\beta)$ space into two subregions where the residual $\textbf{RMSE}$ of  the $\ell_0$ ML decoding (from (\ref{eq:ex7})) is small (i.e., proportional to the noise) or large (i.e., an order of magnitude larger than the noise). These errors are achievable by \emph{an} $\ell_0$ algorithm, i.e., by \emph{an} algorithm that solves (\ref{eq:ex7}) (such an algorithm need \emph{not} necessarily be fast) (see Figure \ref{fig:fig7}).
 \item We also uncover the location of dPT curve  which does the same as aPT while additionally assuming that the utilized solver of (\ref{eq:ex7}) is of \emph{descending} type. We say that dPT curve separates the $(\alpha,\beta)$ regions where small (i.e., proportional to noise) $\textbf{RMSE}$ is achievable by \emph{descending} $\ell_0$ (for short, $d\ell_0$) algorithms from those where it is not (see Figure \ref{fig:fig7}). In a range of small under-samplings dPT outperforms $\ell_1$ hinting that convex relaxations based methods might even be universally beatable.
     \item We implement a simple $d\ell_0$ (see Algorithm \ref{alg:algx1}) and observe that its simulated performance is in an excellent agreement with the theoretical predictions (obtained in an infinite dimensional context) even for dimensions of the order of 100 (see Figures \ref{fig:figprac1} and \ref{fig:figprac2}). We are unaware of any algorithmic structure that can achieve such a performance.
 \item For Fl RDT to be practically relevant, one needs to conduct a sizeable set of numerical evaluations. After conducting all of them, we uncover that the lifting mechanism converges astonishingly fast with relative  corrections in the estimated quantities not exceeding $\sim 0.1\%$ already
     on the third level of lifting (see Table \ref{tab:tab3} as well as Figures \ref{fig:fig5} and \ref{fig:fig6}).
\item We uncover a remarkable set of closed form analytical relations that explicitly connect key lifting parameters.  In addition to being extremely helpful and ultimately of key importance in conducting all of the required numerical work, these relations also provide direct insights into a beautiful structuring of the parametric interconnections (see Corollaries \ref{cor:closedformrel1} and \ref{cor:3closedformrel1}).
\end{itemize}

\subsection{Technical preliminaries}
 \label{sec:mathprelim}

It is rather clear from the preceding discussion that our main concern in what follows will be the performance analysis of the $\ell_0$ ML decoder/decompressor (\ref{eq:ex7}). To that end we introduce two quantities that will play a key role in conducting such an analysis and ultimately characterizing the performance of the estimator $\hat{\x}$ from (\ref{eq:ex7}). The first one is the \emph{objective} value
\begin{eqnarray}
\hspace{-0in}\textbf{\emph{\bl{objective:}}} \qquad \qquad \qquad\xi\triangleq \lim_{n\rightarrow\infty}\frac{\min_{\x,\|\x\|_0=k} \|\y-A\x\|_2}{\sqrt{n}\sigma}, \label{eq:ex8}
\end{eqnarray}
whereas the second one is the ($\sigma$ scaled) \emph{root mean square error} of the ML estimator
\begin{eqnarray}
\hspace{-1in}\textbf{\emph{\bl{root mean square error:}}} \qquad \qquad \qquad \textbf{RMSE}  \triangleq \lim_{n\rightarrow\infty}\frac{\|\bar{\x}-\hat{\x}\|_2}{\sigma}. \label{eq:ex9}
\end{eqnarray}
Many other performance characterizing quantities can be defined and studied as well. For example, fractions of correctly/incorrectly detected support and associated notions of distortion are of interest in certain scenarios (see, e.g., \cite{ReevesG12,ReevesG13} and references therein). They are also closely related to the $\textbf{RMSE}$. In some situations (say, \emph{a priori} known binary signals studied in, e.g., \cite{Stojnicclupint19}), the relation is such that they are even directly computable from each other. A key connecting parameter is the overlap between the true and the estimated signal. As that very same parameter is of critical importance for what we discuss below, we term it as \emph{true-estimated signal overlap} and define it in the following way
\begin{eqnarray}
\hspace{-1in}\textbf{\emph{\bl{true-estimated signal overlap:}}} \qquad \qquad \qquad  \hat{c}_1 & \triangleq & \lim_{n\rightarrow\infty} \bar{\x}^T\hat{\x}. \label{eq:ex10}
\end{eqnarray}
Recalling on (\ref{eq:ex1}), (\ref{eq:ex8}) can be rewritten as
\begin{eqnarray}
\xi &  =  &  \lim_{n\rightarrow\infty} \frac{\min_{\x,\|\x\|_0=k} \|\y-A\x\|_2}{\sqrt{n}\sigma}
=
\lim_{n\rightarrow\infty} \frac{\min_{\x,\|\x\|_0=k} \|A\bar{\x}+\sigma\v-A\x\|_2}{\sqrt{n}\sigma} \nonumber \\
& = &
\lim_{n\rightarrow\infty} \frac{\min_{\x,\|\x\|_0=k} \left \|\begin{bmatrix}
                                         A & \v
                                       \end{bmatrix} \begin{bmatrix}
                                                       \bar{\x}-\x \\
                                                       \sigma
                                                     \end{bmatrix}\right \|_2}{\sqrt{n}\sigma}. \label{eq:ex11}
\end{eqnarray}
 After setting
\begin{eqnarray}
 G\triangleq \begin{bmatrix}
                                         A & \v
                                       \end{bmatrix} , \label{eq:ex12}
\end{eqnarray}
and defining set $\cX(\sigma,c_1,r_1)$
\begin{eqnarray}
\cX(\sigma,c_1,r_1) \triangleq \left \{\tilde{\x}  | \quad \x\in\mR^{n+1},\tilde{\x}_{1:n}=\bar{\x}-\x,\|\x\|_2=r_1,\bar{\x}^T\x=c_1,\|\x\|_0=k, \tilde{\x}_{n+1}=\sigma\right \}, \label{eq:ex13}
\end{eqnarray}
we can further rewrite (\ref{eq:ex11}) as
\begin{equation}
\xi
=
\lim_{n\rightarrow\infty}\frac{\min_{\x,\|\x\|_0=k} \left \|\begin{bmatrix}
                                         A & \v
                                       \end{bmatrix} \begin{bmatrix}
                                                       \bar{\x}-\x \\
                                                       \sigma
                                                     \end{bmatrix}\right \|_2}{\sqrt{n}\sigma}
                                                     =
                         \lim_{n\rightarrow\infty}
                                                     \frac{\min_{c_1,r\geq 0}\min_{\tilde{\x}\in\cX(\sigma,c_1,r_1)} \left \|G \tilde{\x}\right \|_2}{\sqrt{n}\sigma}. \label{eq:ex14}
\end{equation}
To make the presentation neater, in what follows we take $A$ and consequently  $G$ as matrices with iid standard normal components (concentrations and Lindeberg central limit theorem ensure that our final results can quickly be adapted to other statistical setups as well; a particularly quick and elegant presentation of the  Lindeberg principle can be found in e.g., \cite{Chatterjee06}).

Denoting by $\mS^m$ the unit $m$-dimensional sphere, we find it useful to introduce
\begin{equation}
\xi_1(c_1,r_1)\triangleq
                         \lim_{n\rightarrow\infty}
                                                     \frac{\min_{\tilde{\x}\in\cX(\sigma,c_1,r_1)} \left \|G \tilde{\x}\right \|_2}{\sqrt{n}\sigma}
                                                     =
                                                                              \lim_{n\rightarrow\infty}
                                                     \frac{\min_{\tilde{\x}\in\cX(\sigma,c_1,r_1)} \min_{\y\in\mS^m}\y^TG \tilde{\x}}{\sqrt{n}\sigma}, \label{eq:ex16}
\end{equation}
which together with concentrations and (\ref{eq:ex14}) gives
\begin{equation}
\xi  =  \min_{c_1,r_1\geq 0} \xi_1(c_1,r_1). \label{eq:ex17}
\end{equation}
and
\begin{equation}
\mE_{A,\v}\xi  =\mE_G\xi  =  \min_{c_1,r_1\geq 0}\mE_G \xi_1(c_1,r_1). \label{eq:ex18}
\end{equation}
Along the same lines, we also define
\begin{equation}
\delta \triangleq   \mE(\textbf{RMSE} ) = \lim_{n\rightarrow\infty}\frac{\mE\|\bar{\x}-\hat{\x}\|_2}{\sigma}. \label{eq:ex18a0}
\end{equation}
 The adopted convention throughout the paper is that the subscripts next to $\mE$ (or $\mP$) denote the randomness with respect to which the statistical evaluation is taken. When the underlying randomness is clear from the context the subscripts are left unspecified.

\subsection{Connection to free energies}
 \label{sec:feeeng}

To ensure the main results presentation's smoothness, we find it useful to establish a connection between the introduced ML concepts and a corresponding statistical mechanics quantity called free energy. Since $\xi_1(c_1,r)$ plays an important role and is a key object for further studying, the representation from (\ref{eq:ex16}) turns out to be quite relevant. As recognized in the analysis of various random feasibility problems
\cite{StojnicGardGen13,Stojnicnegsphflrdt23}, studying formulation from (\ref{eq:ex16}), practically corresponds to studying properly defined \emph{free energies}. The contours of such a correspondence are for the completeness sketched below (related more detailed discussions can be found in, e.g., \cite{StojnicGardGen13,Stojnicnegsphflrdt23} and references therein).

We start with the following \emph{bilinear Hamiltonian}
\begin{equation}
\cH_{sq}(G)= \y^TG\x,\label{eq:ham1}
\end{equation}
and its corresponding (reciprocal form of) partition function
\begin{equation}
Z_{sq}(\beta,G)=\sum_{\x\in\cX} \lp \sum_{\y\in\cY}e^{\beta\cH_{sq}(G)} \rp^{-1},  \label{eq:partfun}
\end{equation}
where, $\cX$ and $\cY$ are for the time being assumed to be general sets. The thermodynamic limit of the associated (reciprocal form) \emph{average} free energy is then given as
\begin{eqnarray}
f_{sq}(\beta) & = & - \lim_{n\rightarrow\infty}\frac{\mE_G\log{(Z_{sq}(\beta,G)})}{\beta \sqrt{n}}
=\lim_{n\rightarrow\infty} \frac{\mE_G\log\lp \sum_{\x\in\cX} \lp \sum_{\y\in\cY}e^{\beta\cH_{sq}(G)} \rp^{-1} \rp}{\beta \sqrt{n}} \nonumber \\
& = &  - \lim_{n\rightarrow\infty} \frac{\mE_G\log\lp \sum_{\x\in\cX} \lp \sum_{\y\in\cY}e^{\beta\y^TG\x)}\rp^{-1} \rp}{\beta \sqrt{n}}.\label{eq:logpartfunsqrt}
\end{eqnarray}
Particularly focusing on  ``zero-temperature'' ($T\rightarrow 0$ or  $\beta=\frac{1}{T}\rightarrow\infty$) regime, one obtains the so-called (average) \emph{ground-state} energy
\begin{eqnarray}
f_{sq}(\infty)    \triangleq    \lim_{\beta\rightarrow\infty}f_{sq}(\beta) & = &
 - \lim_{\beta,n\rightarrow\infty}\frac{\mE_G\log{(Z_{sq}(\beta,G)})}{\beta \sqrt{n}}
 \nonumber \\
& = &
-  \lim_{n\rightarrow\infty}\frac{\mE_G \max_{\x\in\cX} -  \max_{\y\in\cY} \y^TG\x}{\sqrt{n}}
=
 \lim_{n\rightarrow\infty}\frac{\mE_G \min_{\x\in\cX}  \max_{\y\in\cY} \y^TG\x}{\sqrt{n}}. \nonumber \\
  \label{eq:limlogpartfunsqrta0}
\end{eqnarray}
An obvious direct relation between (\ref{eq:limlogpartfunsqrta0}) and (\ref{eq:ex16})  implies that $f_{sq}(\infty)$ is very tightly connected to $\xi_1(c_1,r)$, which, together with (\ref{eq:ex17}), implies that understanding $f_{sq}(\infty)$'s  behavior is of critical importance for characterizing the ML decoder (\ref{eq:ex7}). Before we get to the technicalities of the following sections which shed a bit more light on $f_{sq}(\infty)$, we should point out that  studying $f_{sq}(\infty)$ directly is typically not very easy. To circumvent such an obstacle, we first study $f_{sq}(\beta)$ assuming a general $\beta$  and then specialize the obtained results to $\beta\rightarrow\infty$ ground state regime. To further facilitate presentation, we, on occasion, neglect analytical details that remain without much ground state relevance.

\section{Utilization of bli random processes and sfl RDT}
\label{sec:randlincons}

The observation that the free energy from (\ref{eq:logpartfunsqrt}),
\begin{eqnarray}
f_{sq}(\beta) & = &\lim_{n\rightarrow\infty} \frac{\mE_G\log\lp \sum_{\x\in\cX} \sum_{\y\in\cY}e^{\beta\y^TG\x)}\rp}{\beta \sqrt{n}},\label{eq:hmsfl1}
\end{eqnarray}
 can be viewed as a function of \emph{bilinearly indexed} (bli) random process $\y^TG\x$ is  of key importance for everything that follows. Such an observation enables connecting $f_{sq}(\beta)$ and studies of bli processes from \cite{Stojnicsflgscompyx23,Stojnicnflgscompyx23,Stojnicflrdt23}. To fully establish such a connection, we closely follow the path traced in \cite{Stojnichopflrdt23,Stojnicnegsphflrdt23} and start with a collection of useful technical definitions. For $r\in\mN$, $k\in\{1,2,\dots,r+1\}$, sets $\cX\subseteq \mR^{n+1}$ and $\cY\subseteq \mR^m$, function $f_S(\cdot):\mR^{n+1}\rightarrow R$, real scalars $x$, $y$, and $s$  such that $x>0$, $y>0$, and $s^2=1$,   vectors $\p=[\p_0,\p_1,\dots,\p_{r+1}]$, $\q=[\q_0,\q_1,\dots,\q_{r+1}]$, and $\c=[\c_0,\c_1,\dots,\c_{r+1}]$ such that
 \begin{eqnarray}\label{eq:hmsfl2}
1=\p_0\geq \p_1\geq \p_2\geq \dots \geq \p_r\geq \p_{r+1} & = & 0 \nonumber \\
1=\q_0\geq \q_1\geq \q_2\geq \dots \geq \q_r\geq \q_{r+1} & = &  0,
 \end{eqnarray}
$\c_0=1$, $\c_{r+1}=0$, and ${\mathcal U}_k\triangleq [u^{(4,k)},\u^{(2,k)},\h^{(k)}]$  with  $u^{(4,k)}\in\mR$, $\u^{(2,k)}\in\mR^m$, and $\h^{(k)}\in\mR^{n+1}$ having iid standard normal elements, we set
  \begin{eqnarray}\label{eq:fl4}
\psi_{S,\infty}(f_{S},\calX,\calY,\p,\q,\c,x,y,s)  =
 \mE_{G,{\mathcal U}_{r+1}} \frac{1}{n\c_r} \log
\lp \mE_{{\mathcal U}_{r}} \lp \dots \lp \mE_{{\mathcal U}_3}\lp\lp\mE_{{\mathcal U}_2} \lp \lp Z_{S,\infty}\rp^{\c_2}\rp\rp^{\frac{\c_3}{\c_2}}\rp\rp^{\frac{\c_4}{\c_3}} \dots \rp^{\frac{\c_{r}}{\c_{r-1}}}\rp, \nonumber \\
 \end{eqnarray}
where
\begin{eqnarray}\label{eq:fl5}
Z_{S,\infty} & \triangleq & e^{D_{0,S,\infty}} \nonumber \\
 D_{0,S,\infty} & \triangleq  & \max_{\x\in\cX,\|\x\|_2=x} s \max_{\y\in\cY,\|\y\|_2=y}
 \lp \sqrt{n} f_{S}
+\sqrt{n}  y    \lp\sum_{k=2}^{r+1}c_k\h^{(k)}\rp^T\x
+ \sqrt{n} x \y^T\lp\sum_{k=2}^{r+1}b_k\u^{(2,k)}\rp \rp \nonumber  \\
 b_k & \triangleq & b_k(\p,\q)=\sqrt{\p_{k-1}-\p_k} \nonumber \\
c_k & \triangleq & c_k(\p,\q)=\sqrt{\q_{k-1}-\q_k}.
 \end{eqnarray}
The preceding technical preliminaries are sufficient to proceed by recalling on the following theorem -- undoubtedly, one of most significant sfl RDT's  components.
\begin{theorem} \cite{Stojnicflrdt23}
\label{thm:thmsflrdt1}  Consider large $n$ linear regime with  $\alpha\triangleq \lim_{n\rightarrow\infty} \frac{m}{n}$, remaining constant as  $n$ grows. Let $\cX\subseteq \mR^{n+1}$ and $\cY\subseteq \mR^m$ be two given sets and let the elements of  $G\in\mR^{m\times n}$
 be i.i.d. standard normals. Assume the complete sfl RDT frame from \cite{Stojnicsflgscompyx23} and consider a given function $f(\y):R^m\rightarrow R$. Set
\begin{align}\label{eq:thmsflrdt2eq1}
   \psi_{rp} & \triangleq - \max_{\x\in\cX} s \max_{\y\in\cY} \lp f(\y)+\y^TG\x \rp
   \qquad  \mbox{(\bl{\textbf{random primal}})} \nonumber \\
   \psi_{rd}(\p,\q,\c,x,y,s) & \triangleq    \frac{x^2y^2}{2}    \sum_{k=2}^{r+1}\Bigg(\Bigg.
   \p_{k-1}\q_{k-1}
   -\p_{k}\q_{k}
  \Bigg.\Bigg)
\c_k
  - \psi_{S,\infty}(f(\y),\calX,\calY,\p,\q,\c,x,y,s) \hspace{.03in} \mbox{(\bl{\textbf{fl random dual}})}. \nonumber \\
 \end{align}
Let $\hat{\p_0}\rightarrow 1$, $\hat{\q_0}\rightarrow 1$, and $\hat{\c_0}\rightarrow 1$, $\hat{\p}_{r+1}=\hat{\q}_{r+1}=\hat{\c}_{r+1}=0$, and let the non-fixed parts of $\hat{\p}\triangleq \hat{\p}(x,y)$, $\hat{\q}\triangleq \hat{\q}(x,y)$, and  $\hat{\c}\triangleq \hat{\c}(x,y)$ be the solutions of the following system
\begin{eqnarray}\label{eq:thmsflrdt2eq2}
   \frac{d \psi_{rd}(\p,\q,\c,x,y,s)}{d\p} =  0,\quad
   \frac{d \psi_{rd}(\p,\q,\c,x,y,s)}{d\q} =  0,\quad
   \frac{d \psi_{rd}(\p,\q,\c,x,y,s)}{d\c} =  0.
 \end{eqnarray}
 Then,
\begin{eqnarray}\label{eq:thmsflrdt2eq3}
    \lim_{n\rightarrow\infty} \frac{\mE_G  \psi_{rp}}{\sqrt{n}}
  & = &
\min_{x>0} \max_{y>0} \lim_{n\rightarrow\infty} \psi_{rd}(\hat{\p}(x,y),\hat{\q}(x,y),\hat{\c}(x,y),x,y,s) \qquad \mbox{(\bl{\textbf{strong sfl random duality}})},\nonumber \\
 \end{eqnarray}
where $\psi_{S,\infty}(\cdot)$ is as in (\ref{eq:fl4})-(\ref{eq:fl5}).
 \end{theorem}
\begin{proof}
Follows from the corresponding one proven in \cite{Stojnicflrdt23} in exactly the same way as Theorem 1 in \cite{Stojnicnegsphflrdt23}.
 \end{proof}

While the above theorem holds for generic sets $\cX$ and $\cY$, the following corollary specializes it so to make it fully operational for studying the $\ell_0$ ML decoding which is of our interest here.
\begin{corollary}
\label{cor:cor1}  Assume the setup of Theorem \ref{thm:thmsflrdt1} with $r>0$, $\sigma >0$, $c_1\in [0,\sqrt{r_1}]$, and $\cX=\cX(\sigma,c_1,r_1)$ and $\cY=\mS^m$, where $\cX(\sigma,c_1,r_1)$ is as in (\ref{eq:ex13}) and $\mS^m$ is the unit $m$-dimensional Euclidean sphere. Set $r_2=\sqrt{\sigma^2+(1-2c_1+r_1^2)}$ and
\begin{align}\label{eq:thmsflrdt2eq1a0}
   \psi_{rp} & \triangleq - \max_{\x\in\cX} s \max_{\y\in\cY} \lp \y^TG\x  \rp
   \qquad  \mbox{(\bl{\textbf{random primal}})} \nonumber \\
   \psi_{rd}(\p,\q,\c,r_2,1,s) & \triangleq    \frac{r_2^2}{2}    \sum_{k=2}^{r+1}\Bigg(\Bigg.
   \p_{k-1}\q_{k-1}
   -\p_{k}\q_{k}
  \Bigg.\Bigg)
\c_k
  - \psi_{S,\infty}(0,\calX,\calY,\p,\q,\c,r_2,1,s) \quad \mbox{(\bl{\textbf{fl random dual}})}. \nonumber \\
 \end{align}
Let the non-fixed parts of $\hat{\p}$, $\hat{\q}$, and  $\hat{\c}$ be the solutions of the following system
\begin{eqnarray}\label{eq:thmsflrdt2eq2a0}
   \frac{d \psi_{rd}(\p,\q,\c,r_2,1,s)}{d\p} =  0,\quad
   \frac{d \psi_{rd}(\p,\q,\c,r_2,1,s)}{d\q} =  0,\quad
   \frac{d \psi_{rd}(\p,\q,\c,r_2,1,s)}{d\c} =  0.
 \end{eqnarray}
 Then,
\begin{eqnarray}\label{eq:thmsflrdt2eq3a0}
    \lim_{n\rightarrow\infty} \frac{\mE_G  \psi_{rp}}{\sqrt{n}}
  & = &
 \lim_{n\rightarrow\infty} \psi_{rd}(\hat{\p},\hat{\q},\hat{\c},r_2,1,s) \qquad \mbox{(\bl{\textbf{strong sfl random duality}})},\nonumber \\
 \end{eqnarray}
where $\psi_{S,\infty}(\cdot)$ is as in (\ref{eq:fl4})-(\ref{eq:fl5}).
 \end{corollary}
\begin{proof}
Follows immediately as a direct consequence of Theorem \ref{thm:thmsflrdt1}, after choosing $f(\y)=0$ and recognizing that the specialized sets $\cX=\cX(\sigma,c_1,r_1)$ and $\cY=\mS^m$ are such that their elements have $r_2=\sqrt{\sigma^2+(1-2c_1+r_1^2)}$ and unit Euclidean norms, respectively.
 \end{proof}

As observed in \cite{Stojnicflrdt23,Stojnichopflrdt23}, due to trivial random primal concentrations, various  probabilistic variants of (\ref{eq:thmsflrdt2eq3}) and (\ref{eq:thmsflrdt2eq3a0}) are possible as well. As these are rather simple and bring no further conceptual insights, we skip stating them and instead work with the expected values throughout the presentation.

\section{Concrete implementation}
\label{sec:prac}

While Theorem \ref{thm:thmsflrdt1} and Corollary \ref{cor:cor1} are very powerful, their power comes to practical fruition only if each of the associated quantities can be evaluated. Two technical obstacles require a particular attention: \textbf{\emph{(i)}} It is not immediately clear how the right value for $r$ should be chosen; and \textbf{\emph{(ii)}} Sets $\cX$ and $\cY$ have no typically desirable component-wise structural characterizations which questions the straightforwardness of decoupling over $\x$ and $\y$. As we see below, each of these potential obstacles turns out to be surpassable.

Assuming the above discussed sets specializations, $\cX=\cX(\sigma,c_1,r_1)$ and $\cY=\mS^m$, and utilizing  Corollary \ref{cor:cor1}, we start by noting that the object of key practical relevance is the following so-called \emph{random dual}
\begin{align}\label{eq:prac1}
    \psi_{rd}(\p,\q,\c,r_2,1,s) & \triangleq    \frac{r_2^2}{2}    \sum_{k=2}^{r+1}\Bigg(\Bigg.
   \p_{k-1}\q_{k-1}
   -\p_{k}\q_{k}
  \Bigg.\Bigg)
\c_k
  - \psi_{S,\infty}(0,\cX,\cY,\p,\q,\c,r_2,1,s) \nonumber \\
  & =   \frac{r_2^2}{2}    \sum_{k=2}^{r+1}\Bigg(\Bigg.
   \p_{k-1}\q_{k-1}
   -\p_{k}\q_{k}
  \Bigg.\Bigg)
\c_k
  - \psi_{S,\infty}(0,\cX(\sigma,c_1,r_1),\mS^m,\p,\q,\c,r_2,1,s) \nonumber \\
  & =   \frac{r_2^2}{2}    \sum_{k=2}^{r+1}\Bigg(\Bigg.
   \p_{k-1}\q_{k-1}
   -\p_{k}\q_{k}
  \Bigg.\Bigg)
\c_k
  - \frac{1}{n}\varphi(D^{(per)}(s),
  \c) - \frac{1}{n}\varphi(D^{(sph)}(s),\c), \nonumber \\
  \end{align}
where, based on (\ref{eq:fl4})-(\ref{eq:fl5}),
  \begin{eqnarray}\label{eq:prac2}
\varphi(D,\c) & \triangleq &
 \mE_{G,{\mathcal U}_{r+1}} \frac{1}{\c_r} \log
\lp \mE_{{\mathcal U}_{r}} \lp \dots \lp \mE_{{\mathcal U}_3}\lp\lp\mE_{{\mathcal U}_2} \lp
\lp    e^{D}   \rp^{\c_2}\rp\rp^{\frac{\c_3}{\c_2}}\rp\rp^{\frac{\c_4}{\c_3}} \dots \rp^{\frac{\c_{r}}{\c_{r-1}}}\rp, \nonumber \\
  \end{eqnarray}
and
\begin{eqnarray}\label{eq:prac3}
D^{(per)}(s) & = & \max_{\x\in\cX(\sigma,c_1,r_1)} \lp   s\sqrt{n}      \lp\sum_{k=2}^{r+1}c_k\h^{(k)}\rp^T\x  \rp \nonumber \\
  D^{(sph)}(s) & \triangleq  &   s \max_{\y\in\mS^m}
\lp  \sqrt{n} r_2 \y^T\lp\sum_{k=2}^{r+1}b_k\u^{(2,k)}\rp \rp.
 \end{eqnarray}

\noindent  \underline{\red{\textbf{\emph{(i) Handling $D^{(per)}(-1)$:}}}}  We start by first recalling on  the $\cX(\sigma,c_1,r_1)$ parametrization from (\ref{eq:ex13})
\begin{eqnarray}
\cX(\sigma,c_1,r_1) = \left \{\tilde{\x}  | \quad \x\in\mR^{n},\tilde{\x}_{1:n}=\bar{\x}-\x,\|\x\|_2=r_1,\bar{\x}^T\x=c_1,\|\x\|_0=k, \tilde{\x}_{n+1}=\sigma\right \}, \label{eq:prac0a1}
\end{eqnarray}
and observing that, due to statistical nature of $D^{(per)}(-1)$ and rotational invariance of $\h^{(k)}$, one can write
\begin{eqnarray}\label{eq:prac4}
D^{(per)}(-1) & = & \max_{\tilde{\x}\in\cX(\sigma,\c_1,r_1)}   \lp -\sqrt{n}      \lp\sum_{k=2}^{r+1}c_k\h^{(k)}\rp^T\tilde{\x} \rp
=
\max_{\tilde{\x}\in\cX(\sigma,\c_1,r_1)}   \lp -\sqrt{n}      \lp -\sum_{k=2}^{r+1}c_k\h^{(k)}\rp^T\tilde{\x} \rp.
 \end{eqnarray}
Utilizing the sparsity imposing mechanism introduced in \cite{StojnicLiftStrSec13}, we reparameterize
 $\cX(\sigma,c_1,r_1)$  as
 \begin{equation}
\cX(\sigma,c_1,r_1) =\left \{\tilde{\x}  | \hspace{.031in} \x\in\mR^{n},\tilde{\x}_{1:n}=\bar{\x}-\d\circ\x,\|\d\circ\x\|_2=r_1,\bar{\x}^T(\d\circ\x)=c_1,\|\d\|_0=k, \d\in\{0,1\}^n,\tilde{\x}_{n+1}=\sigma\right \}. \label{eq:prac0a2}
\end{equation}
 A combination of (\ref{eq:prac4}) and (\ref{eq:prac0a2}) then gives the following optimization form
\begin{eqnarray}\label{eq:prac4a0a2}
D^{(per)}(-1)  =   -\sqrt{n} \min_{\d,\x}  & &  \lp      \lp -\sum_{k=2}^{r+1}c_k\h^{(k)}\rp^T
\begin{bmatrix}
  \bar{\x}-\d\circ\x \\
  \sigma
\end{bmatrix}
\rp
\nonumber \\
\mbox{subject to} & & \|\d\circ\x\|_2=r_1,\bar{\x}^T(\d\circ\x)=c_1,\|\d\|_0=k, \d\in\{0,1\}^n.
 \end{eqnarray}
One can further write the Lagrangian and utilize the strong Lagrangian duality to obtain
\begin{eqnarray}\label{eq:prac4a0a3}
D^{(per)}(-1) & =  &  -\sqrt{n} \min_{\d\in \left \{0,1\right \}^n,\x} \max_{\gamma,\nu_1,\nu}
{\cal L}(\gamma,\mu,\nu)
=
-\sqrt{n} \max_{\gamma,\nu_1,\nu}
\min_{\d\in \left \{0,1\right \}^n,\x}
{\cal L}(\gamma,\mu,\nu),
 \end{eqnarray}
where
\begin{eqnarray}\label{eq:prac4a0a3a0}
{\cal L}(\gamma,\mu,\nu) =     \lp -\sum_{k=2}^{r+1}c_k\h^{(k)}\rp^T
\begin{bmatrix}
  \bar{\x}-\d\circ\x \\
  \sigma
\end{bmatrix}
    +\nu_1 \bar{\x}^T(\d\circ\x)
    - \nu_1 c_1
    +\gamma \|\d\circ\x\|_2^2 -\gamma r_1^2
    +\nu \sum_{i=1}^{n}\d_i -\nu k.
 \end{eqnarray}
 Setting
 \begin{eqnarray}
  \bar{\h}\triangleq -\sum_{k=2}^{r+1}c_k\h^{(k)},\label{eq:prac4a0a3a1}
   \end{eqnarray}
and transforming further, we find
\begin{eqnarray}\label{eq:prac4a0a3a2}
{\cal L}(\gamma,\mu,\nu) &  = &       \bar{\h}^T
\begin{bmatrix}
  \bar{\x}\\
  \sigma
\end{bmatrix}
-
\bar{\h}^T (\d\circ\x)
    +\nu_1 \sum_{i=1}^{n}\bar{\x}_i^T\d_i\x_i  - \nu_1 c_1
    +\gamma \sum_{i=1}^{n}\d_i^2\x_i^2 -\gamma r_1^2
    +\nu \sum_{i=1}^{n}\d_i -\nu k
    \nonumber \\
    & = &
    \sum_{i=1}^{n}\bar{\h}_i\bar{\x}_i +\bar{\h}_{n+1}\sigma
  -
\sum_{i=1}^{n}\bar{\h}_i \d_i\x_i
    +\nu_1 \sum_{i=1}^{n}\bar{\x}_i^T\d_i\x_i  - \nu_1 c_1
    +\gamma \sum_{i=1}^{n}\d_i^2\x_i^2 -\gamma r_1^2
    +\nu \sum_{i=1}^{n}\d_i -\nu k. \nonumber \\
 \end{eqnarray}
To optimize over $\x$, we find
 \begin{eqnarray}
 \frac{d{\cal L} (\gamma,\nu_1,\nu)}{d\x_i}=
   -
 \bar{\h}_i \d_i
    +\nu_1  \bar{\x}_i^T\d_i
    +2\gamma  \d_i^2\x_i.
 \label{eq:prac4a0a3a3}
   \end{eqnarray}
Equaling the above derivative to zero for $\d_i=1$ gives
 \begin{eqnarray}
\x_i= \frac{\bar{\h}_i \d_i
    -\nu_1  \bar{\x}_i^T\d_i}{2\gamma  \d_i^2}
    =
     \frac{\bar{\h}_i
    -\nu_1  \bar{\x}_i^T}{2\gamma }.
 \label{eq:prac4a0a3a4}
   \end{eqnarray}
Combining (\ref{eq:prac4a0a3a2})-(\ref{eq:prac4a0a3a4})  one arrives at
\begin{eqnarray}\label{eq:prac4a0a3a5}
\min_{\d\in\{0,1\}^n,\x}{\cal L}(\gamma,\mu,\nu)
    & = &
    \sum_{i=1}^{n}\bar{\h}_i\bar{\x}_i +\bar{\h}_{n+1}\sigma
  +
\sum_{i=1}^{n} \min \lp- \frac{(\bar{\h}_i
    -\nu_1  \bar{\x}_i)^2}{4\gamma}+\nu,0\rp - \nu_1 c_1
  -\gamma r_1^2
     -\nu k. \nonumber \\
 \end{eqnarray}
 After the appropriate scalings, $\gamma\rightarrow \gamma\sqrt{n}$, $\nu_1\rightarrow\nu_1\sqrt{n}$, and $\nu\rightarrow\frac{\nu}{\sqrt{n}}$, a combination of (\ref{eq:prac4a0a3}) and (\ref{eq:prac4a0a3a5}) gives
\begin{align}\label{eq:prac4a0a6}
D^{(per)}(-1) & =
-\sqrt{n} \max_{\gamma,\nu_1,\nu}
\min_{\d\in \left \{0,1\right \}^n,\x}
{\cal L}(\gamma,\mu,\nu)
\nonumber  \\
& =  \hspace{-.0in}
-  \max_{\gamma,\nu_1,\nu}
    \sum_{i=1}^{n}\bar{\h}_i\bar{\x}_i\sqrt{n} +\bar{\h}_{n+1}\sigma\sqrt{n}
  +
\sum_{i=1}^{n} \min \lp- \frac{(\bar{\h}_i
    -\nu_1  \bar{\x}_i\sqrt{n})^2}{4\gamma}+\nu,0\rp - \nu_1 c_1 n
  -\gamma  r_1^2 n
     -\nu k
     \nonumber \\
     & =  \hspace{-.0in}
-  \max_{\gamma,\nu_1,\nu}
   \lp      - \nu_1 c_1 n
  -\gamma  r_1^2 n
     -\nu k + \bar{\h}_{n+1}\sigma\sqrt{n}
+  \sum_{i=1}^{n} D^{(per)}_i(c_k) \rp\nonumber \\
     & =  \hspace{-.0in}
-  \max_{\gamma,\nu_1,\nu}
   \lp      - \nu_1 c_1 n
  -\gamma  r_1^2 n
     -\nu k + \lp -\sum_{k=2}^{r+1}c_k\h_{n+1}^{(k)}\rp\sigma\sqrt{n}
+  \sum_{i=1}^{n} D^{(per)}_i(c_k) \rp,
 \end{align}
where
\begin{eqnarray}\label{eq:prac4a0a7}
D_i^{(per)}(c_k)
& =  &
 \bar{\h}_i\bar{\x}_i\sqrt{n} +  \min \lp- \frac{(\bar{\h}_i
    -\nu_1  \bar{\x}_i\sqrt{n})^2}{4\gamma}+\nu,0\rp
    \nonumber \\
    & =  &
 \lp -\sum_{k=2}^{r+1}c_k\h_{i}^{(k)}\rp\bar{\x}_i\sqrt{n} +  \min \lp- \frac{\lp -\sum_{k=2}^{r+1}c_k\h_{i}^{(k)}
    -\nu_1  \bar{\x}_i\sqrt{n}\rp^2}{4\gamma}+\nu,0\rp    \nonumber \\
    & =  &
 \lp -\sum_{k=2}^{r+1}c_k\h_{i}^{(k)}\rp\bar{\x}_i\sqrt{n} +  \min \lp- \frac{\lp \sum_{k=2}^{r+1}c_k\h_{i}^{(k)}
    +\nu_1  \bar{\x}_i\sqrt{n}\rp^2}{4\gamma}+\nu,0\rp.
 \end{eqnarray}
To enable concrete evaluations later on, we without loss of generality  for $k$-sparse $\bar{\x}$ assume that its nonzero components are located in the first $k$ positions, i.e., we assume that $\bar{\x}_i=0,k+1\leq i\leq n$. Consequently, we then distinguish
\begin{eqnarray}\label{eq:prac4a0a7a0}
D_{i1}^{(per)}(c_k)
     & =  &
 \lp -\sum_{k=2}^{r+1}c_k\h_{i}^{(k)}\rp\bar{\x}_i\sqrt{n} +  \min \lp- \frac{\lp \sum_{k=2}^{r+1}c_k\h_{i}^{(k)}
    +\nu_1  \bar{\x}_i\sqrt{n}\rp^2}{4\gamma}+\nu,0\rp.
 \end{eqnarray}
and
\begin{eqnarray}\label{eq:prac4a0a7a1}
D_{i0}^{(per)}(c_k)
     & =  &
   \min \lp- \frac{\lp \sum_{k=2}^{r+1}c_k\h_{i}^{(k)}
     \rp^2}{4\gamma}+\nu,0\rp.
 \end{eqnarray}
It should also be noted that the assumed knowledge of sparsity $k$ is actually not needed. In the above derivations it is only reflected through the term $\nu k$ in (\ref{eq:prac4a0a6}). Everything remains in place if, instead of being the exact value of the sparsity, it is just the running parameter emulating the true value of sparsity.

\noindent \underline{\red{\textbf{\emph{(ii) Handling $D^{(sph)}(1)$:}}}}  A simple observation gives
\begin{eqnarray}\label{eq:prac7}
   D^{(sph)}(-1) & \triangleq  & -  \sqrt{n}r_2  \max_{\y\in\mS^m}
\lp  \y^T\lp\sum_{k=2}^{r+1}b_k\u^{(2,k)}\rp \rp
= -  \sqrt{n}  r_2 \left \| \sum_{k=2}^{r+1}b_k\u^{(2,k)} \right \|_2.
 \end{eqnarray}
On then recognizes that a scaled version of the quantity on the most right hand side has already been characterized in \cite{Stojnicnegsphflrdt23} through the utilization of the \emph{square root trick} introduced on numerous occasions in, e.g., \cite{StojnicLiftStrSec13}. As shown in \cite{Stojnicnegsphflrdt23}, we have
\begin{eqnarray}\label{eq:prac9}
   D^{(sph)}(-1)
   & =  &  - r_2 \min_{\gamma_{sq}} \lp \sum_{i=1}^{m} D_i^{(sph)}(b_k)+\gamma_{sq}n \rp, \nonumber \\
 \end{eqnarray}
where
\begin{eqnarray}\label{eq:prac10}
   D_i^{(sph)}(b_k)= \frac{\lp  \sum_{k=2}^{r+1}b_k\u_i^{(2,k)} \rp^2}{4\gamma_{sq}}.
 \end{eqnarray}
Recalling on (\ref{eq:ex16}) and (\ref{eq:limlogpartfunsqrta0}) and utilizing concentrations, one then finds
 \begin{eqnarray}
\sigma \mE_{G} \xi(c_1,r_1)
  =
 \lim_{n\rightarrow\infty}\frac{\mE_G \max_{\x\in\cX(\sigma,c_1,r_1)}   \max_{\y\in\mS^m} \y^TG\x}{\sqrt{n}}
 =
    \lim_{n\rightarrow\infty} \frac{\mE_G  \psi_{rp}}{\sqrt{n}}
   =
 \lim_{n\rightarrow\infty} \psi_{rd}(\hat{\p},\hat{\q},\hat{\c},r_2,1,-1).
  \label{eq:negprac11}
\end{eqnarray}
Keeping in mind (\ref{eq:prac1})-(\ref{eq:prac10}), the following theorem that ultimately enables fitting the $\ell_0$ ML decoding within the sfl RDT machinery can be formulated.

\begin{theorem}
  \label{thm:negthmprac1}
  Assume the setup of Theorem \ref{thm:thmsflrdt1} and Corollary \ref{cor:cor1},
  and consider large $n$ linear regime with $\alpha=\lim_{n\rightarrow\infty} \frac{m}{n}$. Set
  \begin{eqnarray}\label{eq:thm2eq1}
\varphi(D,\c) & = &
 \mE_{G,{\mathcal U}_{r+1}} \frac{1}{\c_r} \log
\lp \mE_{{\mathcal U}_{r}} \lp \dots \lp \mE_{{\mathcal U}_3}\lp\lp\mE_{{\mathcal U}_2} \lp
\lp    e^{D}   \rp^{\c_2}\rp\rp^{\frac{\c_3}{\c_2}}\rp\rp^{\frac{\c_4}{\c_3}} \dots \rp^{\frac{\c_{r}}{\c_{r-1}}}\rp, \nonumber \\
  \end{eqnarray}
and
\begin{eqnarray}\label{eq:thm2eq2}
    \bar{\psi}_{rd}(\p,\q,\c,\gamma_{sq},\gamma,\nu_1,\nu,\sigma,c_1,r_1) & \triangleq &      \frac{r_2^2}{2}    \sum_{k=2}^{r+1}\Bigg(\Bigg.
   \p_{k-1}\q_{k-1}
   -\p_{k}\q_{k}
  \Bigg.\Bigg)
\c_k
- \frac{\sigma^2}{2}    \sum_{k=2}^{r+1}\Bigg(\Bigg.
   \q_{k-1}
   -\q_{k}
  \Bigg.\Bigg)
\c_k \nonumber \\
& &
 - \nu_1 c_1
  -\gamma  r_1^2
     -\nu \beta
   - (1-\beta)\varphi(D_{10}^{(per)}(-1),\c) - \frac{1}{n}\sum_{i=1}^{k}\varphi(D_{i1}^{(per)}(-1),\c) \nonumber \\
    & &  + \gamma_{sq}r_2  - \alpha\varphi(D_1^{(sph)}(-1),\c),
  \end{eqnarray}
where $D_{i1}^{(per)}(-1)$, $D_{10}^{(per)}(-1)$, and $D_1^{(sph)}(-1)$ are as in (\ref{eq:prac4a0a7a0}), (\ref{eq:prac4a0a7a1}), and (\ref{eq:prac10}), respectively. Let the ``fixed'' parts of $\hat{\p}$, $\hat{\q}$, and $\hat{\c}$ satisfy $\hat{\p}_1\rightarrow 1$, $\hat{\q}_1\rightarrow 1$, $\hat{\c}_1\rightarrow 1$, $\hat{\p}_{r+1}=\hat{\q}_{r+1}=\hat{\c}_{r+1}=0$. Further, let the ``non-fixed'' parts of $\hat{\p}_k$, $\hat{\q}_k$, and $\hat{\c}_k$ ($k\in\{2,3,\dots,r\}$) and $\hat{\gamma}_{sq}$, $\hat{\gamma}$, $\hat{\nu}_1$, and $\hat{\nu}$ be the solutions of the following system of equations
  \begin{eqnarray}\label{eq:negthmprac1eq1}
   \frac{d \bar{\psi}_{rd}(\p,\q,\c,\gamma_{sq},\gamma,\nu_1,\nu,\sigma,c_1,r_1)}{d\p} & =  & 0 \nonumber \\
   \frac{d \bar{\psi}_{rd}(\p,\q,\c,\gamma_{sq},\gamma,\nu_1,\nu,\sigma,c_1,r_1)}{d\q} & =  & 0 \nonumber \\
   \frac{d \bar{\psi}_{rd}(\p,\q,\c,\gamma_{sq},\gamma,\nu_1,\nu,\sigma,c_1,r_1)}{d\c} & =  & 0 \nonumber \\
   \frac{d \bar{\psi}_{rd}(\p,\q,\c,\gamma_{sq},\gamma,\nu_1,\nu,\sigma,c_1,r_1)}{d\gamma_{sq}} & =  & 0\nonumber \\
   \frac{d \bar{\psi}_{rd}(\p,\q,\c,\gamma_{sq},\gamma,\nu_1,\nu,\sigma,c_1,r_1)}{d\gamma} & =  & 0\nonumber \\
   \frac{d \bar{\psi}_{rd}(\p,\q,\c,\gamma_{sq},\gamma,\nu_1,\nu,\sigma,c_1,r_1)}{d\nu_1} & =  & 0\nonumber \\
   \frac{d \bar{\psi}_{rd}(\p,\q,\c,\gamma_{sq},\gamma,\nu_1,\nu,\sigma,c_1,r_1)}{d\nu} & =  &  0.
 \end{eqnarray}
Consequently, let
\begin{eqnarray}\label{eq:prac17}
c_k(\hat{\p},\hat{\q})  & = & \sqrt{\hat{\q}_{k-1}-\hat{\q}_k} \nonumber \\
b_k(\hat{\p},\hat{\q})  & = & \sqrt{\hat{\p}_{k-1}-\hat{\p}_k}.
 \end{eqnarray}
 Then
 \begin{equation}
\sigma\xi(c_1.r_1) =   \bar{\psi}_{rd}(\hat{\p},\hat{\q},\hat{\c},\hat{\gamma}_{sq},\hat{\gamma},\hat{\nu_1},\hat{\nu},\sigma,c_1,r_1).
  \label{eq:negthmprac1eq2}
\end{equation}
\end{theorem}
\begin{proof}
Follows from  Theorem \ref{thm:thmsflrdt1}, Corollary \ref{cor:cor1}, the sfl RDT machinery presented in \cite{Stojnicnflgscompyx23,Stojnicsflgscompyx23,Stojnicflrdt23,Stojnichopflrdt23}, the above discussion, and after integrating out $\lp\sum_{k=2}^{r+1}c_k\h_{n+1}^{(k)}\rp\sigma\sqrt{n}$.
\end{proof}

\noindent \textbf{Remark:} Along the lines of the observation right after (\ref{eq:prac4a0a7a1}), the assumed knowledge of sparsity $k$ is actually not needed.  Everything remains in place if $\beta$ in (\ref{eq:thm2eq2}) (instead of representing the exact value of sparsity ratio) is just viewed as a running parameter emulating the sparsity ratio. Moreover, this also means that the above can be utilized to characterize not only the ML optimal sparsity (if it is a priori unknown) but also to characterize the ML performance for any assumed sparsity. As our prevalent interest here is the $\ell_0$ ML from (\ref{eq:ex7}) (which assumes that the sparsity is a priori known), we below conduct numerical evaluations within such a context.

\subsection{Numerical evaluations}
\label{sec:nuemricalags}

The results of Theorem \ref{thm:negthmprac1} are conceptually very powerful, but, as emphasized earlier, they become practically relevant only if  one can conduct all the underlying numerical evaluations. We start the evaluations with $r=1$ and then proceed systematically by incrementally increasing $r$ and showing how the entire lifting mechanism is progressing. Along the way we uncover several closed form analytical results that substantially simplify and facilitate the entire evaluation process. Also, to be able to provide concrete numerical values, $\bar{\x}$ needs to be precisely specified. We choose $\bar{\x}_i=\frac{1}{\sqrt{k}},1\leq i\leq k$, which is typically viewed as algorithmically the hardest.

\subsubsection{$r=1$ -- first level of lifting}
\label{sec:firstlev}

For the first level, one has $r=1$ and $\hat{\p}_1\rightarrow 1$ and $\hat{\q}_1\rightarrow 1$ which, coupled with $\hat{\p}_{r+1}=\hat{\p}_{2}=\hat{\q}_{r+1}=\hat{\q}_{2}=0$, and $\hat{\c}_{2}\rightarrow 0$, gives
\begin{align}\label{eq:negprac19}
    \bar{\psi}_{rd}^{(1)}    & =   \frac{r_2^2-\sigma^2}{2}
\hat{\c}_2
- \nu_1 c_1
  -\gamma  r_1^2
     -\nu \beta
     - \frac{1-\beta}{\hat{\c}_2}\log\lp \mE_{{\mathcal U}_2} e^{-\hat{\c}_2   \min \lp- \frac{\lp \sqrt{1-0}\h_{i}^{(2)}
     \rp^2}{4\gamma}+\nu,0\rp }\rp
\nonumber \\
& \quad     - \frac{1}{\hat{n\c}_2}\sum_{i=1}^{k}\log\lp \mE_{{\mathcal U}_2} e^{\hat{\c}_2  \lp \lp \sqrt{1-0}\h_{i}^{(2)}\rp\bar{\x}_i\sqrt{n} -  \min \lp- \frac{\lp \sqrt{1-0}\h_{i}^{(2)}
    +\nu_1  \bar{\x}_i\sqrt{n}\rp^2}{4\gamma}+\nu,0\rp \rp}\rp
 \nonumber \\
&  \quad +\gamma_{sq}r_2
- \alpha\frac{1}{\hat{\c}_2}\log\lp \mE_{{\mathcal U}_2} e^{-\hat{\c}_2r_2\frac{\lp  \sqrt{1-0}\u_1^{(2,2)} \rp^2}{4\gamma_{sq}}}\rp \nonumber \\
& \rightarrow
- \nu_1 c_1
  -\gamma  r_1^2
     -\nu \beta
     \nonumber \\
& \quad   -\frac{1-\beta}{\hat{\c}_2}\log\lp 1- \mE_{{\mathcal U}_2} \hat{\c}_2   \lp   \min \lp- \frac{\lp \sqrt{1-0}\h_{i}^{(2)}
    \rp^2}{4\gamma}+\nu,0\rp \rp \rp
     \nonumber \\
& \quad   -\frac{1}{\hat{n\c}_2}\sum_{i=1}^{k}\log\lp 1+ \mE_{{\mathcal U}_2} \hat{\c}_2   \lp \lp \sqrt{1-0}\h_{i}^{(2)}\rp\bar{\x}_i\sqrt{n} -  \min \lp- \frac{\lp \sqrt{1-0}\h_{i}^{(2)}
    +\nu_1  \bar{\x}_i\sqrt{n}\rp^2}{4\gamma}+\nu,0\rp \rp \rp
 \nonumber \\
&  \quad +\gamma_{sq}r_2
- \alpha\frac{1}{\hat{\c}_2}\log\lp 1 - \mE_{{\mathcal U}_2} \hat{\c}_2r_2\frac{\lp  \sqrt{1-0}\u_1^{(2,2)} \rp^2}{4\gamma_{sq}}\rp \nonumber \\
\nonumber \\
& \rightarrow
- \nu_1 c_1
  -\gamma  r_1^2
     -\nu \beta
      \nonumber \\
& \quad   -\frac{1-\beta}{\hat{\c}_2}\log\lp 1- \mE_{{\mathcal U}_2} \hat{\c}_2  \lp    \min \lp- \frac{\lp \sqrt{1-0}\h_{i}^{(2)}
    +\nu_1  \bar{\x}_i\sqrt{n}\rp^2}{4\gamma}+\nu,0\rp \rp \rp
      \nonumber \\
& \quad   -\frac{1}{\hat{n\c}_2}\sum_{i=1}^{k}\log\lp 1+ \mE_{{\mathcal U}_2} \hat{\c}_2  \lp \lp \sqrt{1-0}\h_{i}^{(2)}\rp\bar{\x}_i\sqrt{n} -   \min \lp- \frac{\lp \sqrt{1-0}\h_{i}^{(2)}
    +\nu_1  \bar{\x}_i\sqrt{n}\rp^2}{4\gamma}+\nu,0\rp \rp \rp
 \nonumber \\
 &  \quad +\gamma_{sq}r_2
- \alpha\frac{1}{\hat{\c}_2}\log\lp 1 - \frac{\hat{\c}_2r_2}{4\gamma_{sq}}\rp \nonumber \\
& \rightarrow
- \nu_1 c_1
  -\gamma  r_1^2
     -\nu \beta
+(1-\beta)\mE_{{\mathcal U}_2}   \lp   \min \lp- \frac{\lp \sqrt{1-0}\h_{i}^{(2)}
    +\nu_1  \bar{\x}_i\sqrt{n}\rp^2}{4\gamma}+\nu,0\rp \rp
    \nonumber \\
    & \quad
    -\frac{1}{n}\sum_{i=1}^{k}\mE_{{\mathcal U}_2}   \lp \lp \sqrt{1-0}\h_{i}^{(2)}\rp\bar{\x}_i\sqrt{n} -  \min \lp- \frac{\lp \sqrt{1-0}\h_{i}^{(2)}
    +\nu_1  \bar{\x}_i\sqrt{n}\rp^2}{4\gamma}+\nu,0\rp \rp
    \nonumber \\
& \quad  +\gamma_{sq}r_2
+   r_2\frac{\alpha}{4\gamma_{sq}},
  \end{align}
  where we leave out the arguments and write $ \bar{\psi}_{rd}^{(1)}   $ instead of  $ \bar{\psi}_{rd}^{(1)}(\hat{\p},\hat{\q},\hat{\c},\gamma_{sq},\gamma,\nu_1,\nu,\sigma,c_1,r_1)  $ to simplify notation. Utilizing  $\frac{ \bar{\psi}_{rd}(\hat{\p},\hat{\q},\hat{\c},\gamma_{sq},\gamma,\nu_1,\nu,\sigma,c_1,r_1) }{d\gamma_{sq}}=0$, we easily find
$\hat{\gamma}_{sq}=\frac{\sqrt{\alpha}}{2}$, which, together with (\ref{eq:negprac19}), gives
\begin{eqnarray}\label{eq:negprac19a0}
    \bar{\psi}_{rd}^{(1)}   &  \rightarrow &
- \nu_1 c_1
  -\gamma  r_1^2
     -\nu \beta
+(1-\beta)\mE_{{\mathcal U}_2}   \lp   \min \lp- \frac{\lp \h_{i}^{(2)}
     \rp^2}{4\gamma}+\nu,0\rp \rp
\nonumber \\
& &
-\frac{1}{n}\sum_{i=1}^{k}\mE_{{\mathcal U}_2}   \lp  \h_{i}^{(2)}\bar{\x}_i\sqrt{n} -  \min \lp- \frac{\lp \h_{i}^{(2)}
    +\nu_1  \bar{\x}_i\sqrt{n}\rp^2}{4\gamma}+\nu,0\rp \rp
+ \sqrt{\alpha}r_2 \nonumber \\
&  \rightarrow &
 f_1(\gamma,\nu_1,\nu) + \sqrt{\alpha}r_2,
  \end{eqnarray}
where
\begin{eqnarray}\label{eq:negprac19a1}
f_1(\gamma,\nu_1,\nu)  & \triangleq &
- \nu_1 c_1
  -\gamma  r_1^2
     -\nu \beta
+(1-\beta)\mE_{{\mathcal U}_2}   \lp   \min \lp- \frac{\lp \h_{i}^{(2)}
     \rp^2}{4\gamma}+\nu,0\rp \rp
\nonumber \\
 & & -\frac{1}{n}\sum_{i=1}^{k} \mE_{{\mathcal U}_2}   \lp  \h_{i}^{(2)}\bar{\x}_i\sqrt{n} -  \min \lp- \frac{\lp \h_{i}^{(2)}
   +\nu_1  \bar{\x}_i\sqrt{n}\rp^2}{4\gamma}+\nu,0\rp \rp
   \nonumber \\
   & = &
- \nu_1 c_1
  -\gamma  r_1^2
     -\nu \beta
+(1-\beta)\mE_{{\mathcal U}_2}   \lp   \min \lp- \frac{\lp \h_{i}^{(2)}
     \rp^2}{4\gamma}+\nu,0\rp \rp
\nonumber \\
& &
+\frac{1}{n}\sum_{i=1}^{k} \mE_{{\mathcal U}_2}   \lp  \min \lp- \frac{\lp \h_{i}^{(2)}
   +\nu_1  \bar{\x}_i\sqrt{n}\rp^2}{4\gamma}+\nu,0\rp \rp.
  \end{eqnarray}
After recalling on $\bar{\x}_i=\frac{1}{\sqrt{k}},1\leq i\leq k,$, setting
\begin{eqnarray}\label{eq:negprac19a1a0}
\bar{d} & = & \frac{1}{\sqrt{\beta}} \nonumber \\
\bar{a} &  =  & 2\sqrt{\gamma\nu} \nonumber \\
\bar{a}_1 & = & 2\sqrt{\gamma\nu}-\nu_1\bar{d}\nonumber \\
\bar{a}_2 & = & -2\sqrt{\gamma\nu}- \nu_1\bar{d} \nonumber \\
\bar{b} & = & \nu_1\bar{d},
 \end{eqnarray}
and
\begin{eqnarray}\label{eq:negprac19a1a1}
f_{21} & =  &   2\lp-\frac{1}{4\gamma}\lp\frac{\bar{a}e^{-\bar{a}^2/2} }{\sqrt{2\pi}} + \frac{1}{2}\erfc\lp \frac{\bar{a}}{\sqrt{2}}\rp\rp + \frac{\nu}{2}\erfc\lp \frac{\bar{a}}{\sqrt{2}} \rp \rp
\nonumber \\
 f_{221}  & = &
     \lp-\frac{1}{4\gamma}\lp\frac{(\bar{a}_1+2\bar{b})e^{-\bar{a}_1^2/2} }{\sqrt{2\pi}} + \frac{\bar{b}^2+1}{2}\erfc\lp \frac{\bar{a}_1}{\sqrt{2}}\rp\rp + \frac{\nu}{2}\erfc\lp \frac{\bar{a}_1}{\sqrt{2}} \rp \rp \nonumber \\
 f_{222}  & = &
     \lp-\frac{1}{4\gamma}\lp -\frac{(\bar{a}_2+2\bar{b})e^{-\bar{a}_2^2/2} }{\sqrt{2\pi}} + \frac{\bar{b}^2+1}{2}\erfc\lp -\frac{\bar{a}_2}{\sqrt{2}}\rp\rp + \frac{\nu}{2}\erfc\lp -\frac{\bar{a}_2}{\sqrt{2}} \rp \rp,
 \end{eqnarray}
one computes the integrals and obtains
\begin{eqnarray}\label{eq:negprac19a1a2}
f_1(\gamma,\nu_1,\nu)  =
- \nu_1 c_1
  -\gamma  r_1^2
     -\nu \beta
+(1-\beta) f_{21} +\beta(f_{221}+f_{222}).
  \end{eqnarray}
Connecting (\ref{eq:negprac19a0})  and  (\ref{eq:negprac19a1a2}), we finally have
\begin{eqnarray}\label{eq:negprac19a0b0}
    \bar{\psi}_{rd}^{(1)}    &  \rightarrow &
 f_1(\gamma,\nu_1,\nu) + \sqrt{\alpha}r_2 =- \nu_1 c_1
  -\gamma  r_1^2
     -\nu \beta
+(1-\beta) f_{21} +\beta(f_{221}+f_{222}) + \sqrt{\alpha}r_2,
  \end{eqnarray}
where $f_{21}$, $f_{221}$, and $f_{222}$ are given through  (\ref{eq:negprac19a1a0}), and (\ref{eq:negprac19a1a1}). One then obtains $\hat{\gamma}$, $\hat{\nu}_1$, and $\hat{\nu}$ as the solutions of the following system
\begin{eqnarray}\label{eq:dersystem0}
 \frac{d\bar{\psi}_{rd}^{(1)} }{d\gamma}
 =
  \frac{d\bar{\psi}_{rd}^{(1)} }{d\nu_1}
  =
   \frac{d\bar{\psi}_{rd}^{(1)} }{d\nu}=0.
\end{eqnarray}
To get concrete values we take $\alpha=0.057$, $\beta=0.01$, $c_1=0.5$, $r_1=1$, and $\sigma=\sqrt{0.00025}$ (throughout numerical evaluations we will keep $r_1=1$, as in the systems where the SNR is known it is trivial to estimate the norm of $\bar{\x}$ (which is assumed to be $1$) and optimizing additionally over $r_1$ would mean that one is not utilizing that information). The obtained results for all parameters are given in Table \ref{tab:tab1}.
\begin{table}[h]
\caption{$1$-sfl RDT parameters; $\ell_0$ ML decoding ;  $\hat{\c}_1\rightarrow 1$; $n,\beta\rightarrow\infty$; $c_1=0.5$, $r_1=1$}\vspace{.1in}
\centering
\def\arraystretch{1.2}
\begin{tabular}{||l||c|c|c|c||c|c||c|c||c||c||}\hline\hline
 \hspace{-0in}$r$-sfl RDT                                             & $\hat{\gamma}_{sq}$   & $\hat{\gamma}$    & $\hat{\nu_1}$   & $\hat{\nu}$ &  $\hat{\p}_2$ & $\hat{\p}_1$     & $\hat{\q}_2$  & $\hat{\q}_1$ &  $\hat{\c}_2$    & $\xi_1^{(r)}(c_1,r_1)$  \\ \hline\hline
$1$-sfl RDT                                       & $0.1194$  & $0.1673$   & $-0.2711$ & $11.5934$ & $0$  & $\rightarrow 1$   & $0$ & $\rightarrow 1$
 &  $\rightarrow 0$  & \bl{$\mathbf{2.5149}$}
  \\ \hline\hline
  \end{tabular}
\label{tab:tab1}
\end{table}
 We complement the numerical results for $\xi_1(c_1,r_1)$ from Table \ref{tab:tab1} with the corresponding ones obtained for a range of $c_1$ in Figure \ref{fig:fig1}. As we will see later on, it will turn out that these results can be  improved in a wide range of $c_1$. Nonetheless, several points are of critical importance and need to be addressed carefully. The curve has a very non-monotonic shape which suggests that the connection between the decoding error ($\textbf{RMSE}$) and the value of the objective might not be very straightforward. In particular, a small objective clearly does not always imply a small recovery error (which one might intuitively expect from the ML criterion). Moreover, the curve has multiple local optima. The good thing in this scenario though, is that at least the objective's global minimum seems to be achievable for a fairly small value of error. In general, however, things can be even much worse.
  \begin{figure}[h]
\centering
\centerline{\includegraphics[width=1\linewidth]{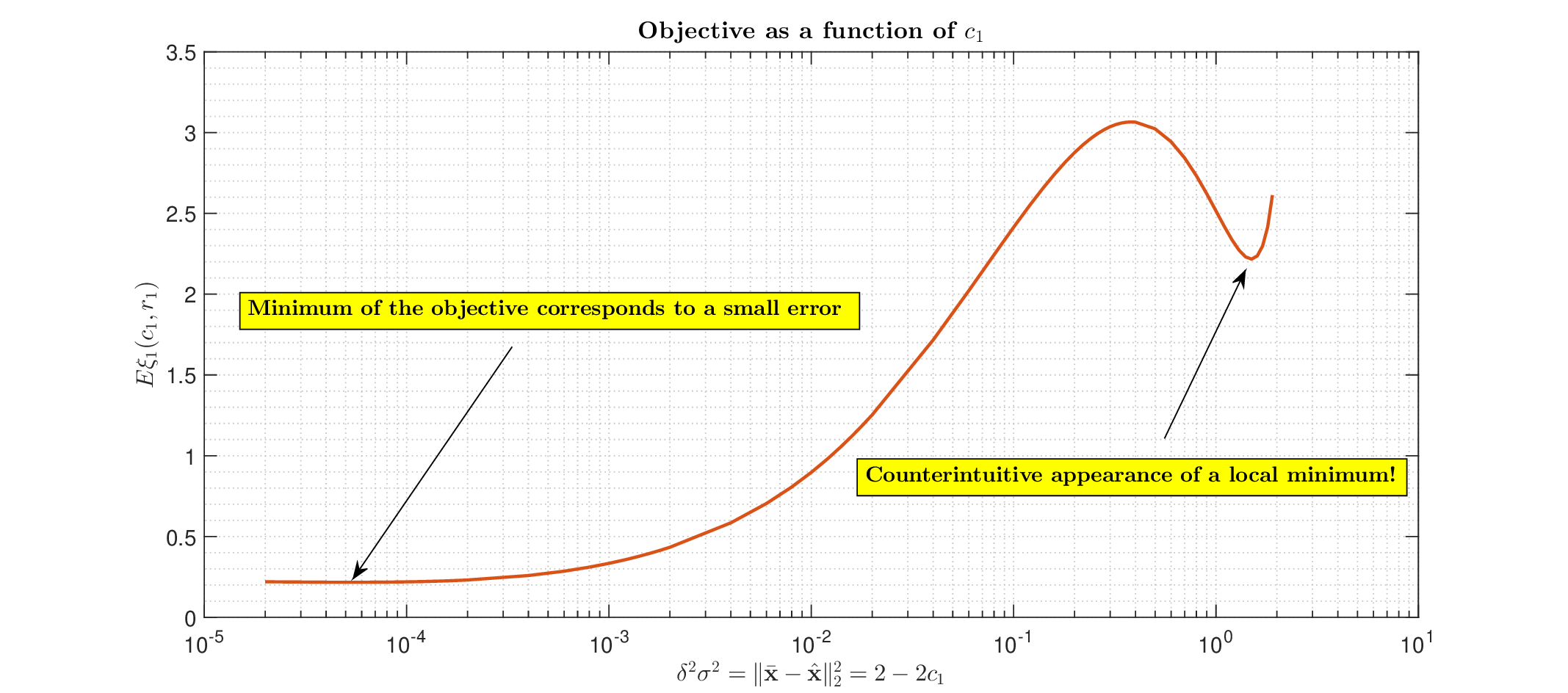}}
\caption{$\xi_1^{(1)}(c_1,r_1)$ as a function of $c_1$; $\alpha=0.057$, $\beta=0.01$, $\sigma=\sqrt{0.00025}$, and $r_1=1$.}
\label{fig:fig1}
\end{figure}
Looking at Figure \ref{fig:fig2}, which shows the same set of results for $\alpha=0.563$ and $\beta=0.4$, one observes that the global minimum is now achieved for a fairly large error. This  would clearly suggest that searching for a minimal objective might not always be the right way to go. A fact somewhat contradictory to the common wisdom behind the ML decoding (particularly in the so-called high $\text{SNR}=1/\sigma^2$ regime). Right now, one can not provide the answers to this conundrum. Instead, they will start appearing  as the presentation progresses. However, it is useful to point out here (and keep in mind for the remainder of the presentation) that the shape of these curves is not mistakenly or unexplainably weird, but is indeed as presented and precisely as such will play a critical role in overall understanding and ultimate utilization of the main results.
\begin{figure}[h]
\centering
\centerline{\includegraphics[width=1\linewidth]{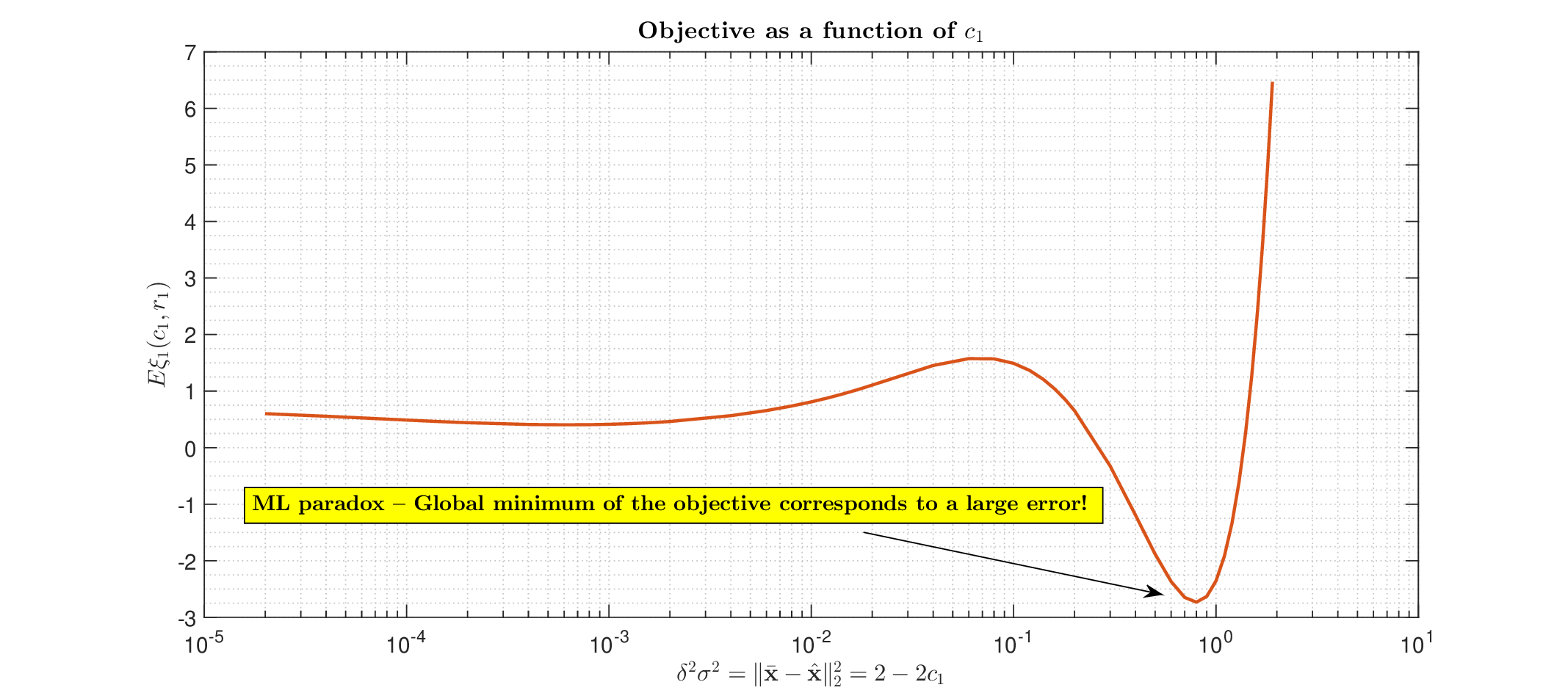}}
\caption{$\xi_1^{(1)}(c_1,r_1)$ as a function of $c_1$; $\alpha=0.563$, $\beta=0.4$, $\sigma=\sqrt{0.00025}$, and $r_1=1$.}
\label{fig:fig2}
\end{figure}

\subsubsection{$r=2$ -- second level of lifting}
\label{sec:secondlev}

The setup presented in the previous subsection can be utilized for the second level as well. One now, however, has $r=2$ and, similarly to what happened on the first level, $\hat{\p}_1\rightarrow 1$ and $\hat{\q}_1\rightarrow 1$. On the other hand, $\hat{\c}_{2}\neq 0$, $\p_2\neq0$, and $\q_2\neq0$ which, together with $\hat{\p}_{r+1}=\hat{\p}_{3}=\hat{\q}_{r+1}=\hat{\q}_{3}=0$ enables us to write, analogously to (\ref{eq:negprac19}),
\begin{eqnarray}\label{eq:negprac24}
    \bar{\psi}_{rd}^{(2)}
    & = &
    \frac{r_2^2}{2}
(1-\p_2\q_2)\c_2 -  \frac{\sigma^2}{2}
(1-\q_2)\c_2
- \nu_1 c_1
  -\gamma  r_1^2
     -\nu \beta
  \nonumber \\
& & - \frac{1-\beta}{\c_2}\mE_{{\mathcal U}_3}\log\lp \mE_{{\mathcal U}_2}
e^{-\c_2   \min \lp- \frac{\lp \sqrt{1-\q_2}\h_{i}^{(2)} + \sqrt{\q_2}\h_{i}^{(3)}
     \rp^2}{4\gamma}+\nu,0\rp }\rp
  \nonumber \\
& & - \frac{1}{n\c_2} \sum_{i=1}^{k}\mE_{{\mathcal U}_3}\log\lp \mE_{{\mathcal U}_2}
e^{\c_2 \lp \lp \sqrt{1-\q_2}\h_{i}^{(2)} + \sqrt{\q_2}\h_{i}^{(3)}\rp\bar{\x}_i\sqrt{n} -  \min \lp- \frac{\lp \sqrt{1-\q_2}\h_{i}^{(2)} + \sqrt{\q_2}\h_{i}^{(3)}
  +\nu_1\bar{\x}\sqrt{n}   \rp^2}{4\gamma}+\nu,0\rp \rp }
\rp
\nonumber \\
& &   + \gamma_{sq}r_2
 -\alpha\frac{1}{\c_2}\mE_{{\mathcal U}_3} \log\lp \mE_{{\mathcal U}_2} e^{-\c_2r_2\frac{\lp\sqrt{1-\p_2}\u_1^{(2,2)}+\sqrt{\p_2}\u_1^{(2,3)}\rp^2}{4\gamma_{sq}}}\rp \nonumber \\
& = &    \frac{r_2^2}{2}
(1-\p_2\q_2)\c_2 -  \frac{\sigma^2}{2}
(1-\q_2)\c_2
- \nu_1 c_1
  -\gamma  r_1^2
     -\nu \beta
  \nonumber \\
& & - \frac{1-\beta}{\c_2}\mE_{{\mathcal U}_3}\log\lp \mE_{{\mathcal U}_2}
e^{-\c_2   \min \lp- \frac{\lp \sqrt{1-\q_2}\h_{i}^{(2)} + \sqrt{\q_2}\h_{i}^{(3)}
     \rp^2}{4\gamma}+\nu,0\rp }\rp
  \nonumber \\
& & - \frac{\beta}{\c_2}\mE_{{\mathcal U}_3}\log\lp \mE_{{\mathcal U}_2}
e^{\c_2 \lp \lp \sqrt{1-\q_2}\h_{i}^{(2)} + \sqrt{\q_2}\h_{i}^{(3)}\rp \bar{d} -  \min \lp- \frac{\lp \sqrt{1-\q_2}\h_{i}^{(2)} + \sqrt{\q_2}\h_{i}^{(3)}
  +\nu_1\bar{d}  \rp^2}{4\gamma}+\nu,0\rp \rp }
\rp
\nonumber \\
 & &  +  \gamma_{sq}r_2
+ \alpha r_2 \Bigg(\Bigg. \frac{1}{2\c_2r_2} \log \lp \frac{2\gamma_{sq}+\c_2r_2(1-\p_2)}{2\gamma_{sq}} \rp  +  \frac{\p_2}{2(2\gamma_{sq}+\c_2r_2(1-\p_2))}   \Bigg.\Bigg)
\nonumber \\
& = &    \frac{r_2^2}{2}
(1-\p_2\q_2)\c_2 -  \frac{\sigma^2}{2}
(1-\q_2)\c_2
- \nu_1 c_1
  -\gamma  r_1^2
     -\nu \beta
  \nonumber \\
& & - \frac{1-\beta}{\c_2}\mE_{{\mathcal U}_3}\log\lp \mE_{{\mathcal U}_2}
{\cal Z}_0^{(2)}\rp
  - \frac{\beta}{\c_2}\mE_{{\mathcal U}_3}\log\lp \mE_{{\mathcal U}_2}
{\cal Z}_0^{(2)}
\rp
\nonumber \\
 & &  +  \gamma_{sq}r_2
+ \alpha r_2 \Bigg(\Bigg. \frac{1}{2\c_2r_2} \log \lp \frac{2\gamma_{sq}+\c_2r_2(1-\p_2)}{2\gamma_{sq}} \rp  +  \frac{\p_2}{2(2\gamma_{sq}+\c_2r_2(1-\p_2))}   \Bigg.\Bigg),
     \end{eqnarray}
where
\begin{eqnarray}
{\cal Z}_0^{(2)} & = &      e^{-\c_2   \min \lp- \frac{\lp \sqrt{1-\q_2}\h_{i}^{(2)} + \sqrt{\q_2}\h_{i}^{(3)}
     \rp^2}{4\gamma}+\nu,0\rp } \nonumber \\
     {\cal Z}_1^{(2)} & = & e^{\c_2 \lp \lp \sqrt{1-\q_2}\h_{i}^{(2)} + \sqrt{\q_2}\h_{i}^{(3)}\rp \bar{d} -  \min \lp- \frac{\lp \sqrt{1-\q_2}\h_{i}^{(2)} + \sqrt{\q_2}\h_{i}^{(3)}
  +\nu_1\bar{d}  \rp^2}{4\gamma}+\nu,0\rp \rp }.
\end{eqnarray}
We then set
\begin{eqnarray}\label{eq:negprac24a0a0}
 D & = & \frac{(-\h_1^{(3)}\sqrt{\q_2}-2\sqrt{\gamma\nu})}{\sqrt{1-\q_2}} \nonumber \\
 E & = & \frac{(-\h_1^{(3)}\sqrt{\q_2}+2\sqrt{\gamma\nu})}{\sqrt{1-\q_2}} \nonumber \\
A & = & \sqrt{\c_2(1-\q_2)} \nonumber \\
B & = & \h_1^{(3)}\sqrt{\c_2\q_2}\nonumber \\
C & = & 4\gamma \nonumber \\
Q_x(A,B,C,D,E) &= & \erf\lp\frac{(-2A^2 E - 2 AB + CE)}{\sqrt{2}\sqrt{C} \sqrt{C - 2A^2}}\rp -
\erf\lp\frac{(-2A^2 D - 2 AB + CD)}{\sqrt{2}\sqrt{C} \sqrt{C - 2A^2}}\rp
\nonumber \\
I_{x1}(A,B,C,D,E) & = &  \frac{\sqrt{C} e^{-\c_2\nu + B^2/(C - 2 A^2)} }{2 \sqrt{C - 2 A^2}}(2-Q_x) \nonumber \\
I_{x2}(D,E) & = & \frac{1}{2}\erfc\lp \frac{D}{\sqrt{2}}\rp - \frac{1}{2}\erfc\lp \frac{E}{\sqrt{2}}\rp,
 \end{eqnarray}
and after solving the remaining integrals obtain
\begin{eqnarray} \label{eq:negprac24a3a0}
  \mE_{{\mathcal U}_2}
{\cal Z}_0^{(2)} =     I_{x1}(A,B,C,D,E)+I_{x2}(D,E),
    \end{eqnarray}
and
\begin{eqnarray} \label{eq:negprac24a3}
\mE_{{\mathcal U}_3}\log\lp \mE_{{\mathcal U}_2}
{\cal Z}_0^{(2)}\rp=   \mE_{{\mathcal U}_3} \log\lp I_{x1}(A,B,C,D,E)+I_{x2}(D,E) \rp.
    \end{eqnarray}
We then analogously  set
\begin{eqnarray}\label{eq:negprac24a0a0a0}
 \bar{D} & = & \frac{(-(\h_1^{(3)}\sqrt{\q_2}+\bar{d}\nu_1) -2\sqrt{\gamma\nu})}{\sqrt{1-\q_2}} \nonumber \\
 \bar{E} & = & \frac{(-(\h_1^{(3)}\sqrt{\q_2}+\bar{d}\nu_1) +2\sqrt{\gamma\nu})}{\sqrt{1-\q_2}} \nonumber \\
\bar{A} & = & \sqrt{\c_2(1-\q_2)} \nonumber \\
\bar{B}_1 & = & \h_1^{(3)}\sqrt{\c_2\q_2} + \sqrt{\c_2}\bar{d}\nu_1\nonumber \\
\bar{C} & = & 4\gamma \nonumber \\
\bar{F} & = & \c_2\bar{d}\sqrt{1-\q_2} \nonumber \\
\bar{B} & = & \frac{\bar{F}\bar{C}}{2\bar{A}}+\bar{B}_1 \nonumber \\
\bar{G} & = &\frac{\bar{B}_1^2-\bar{B}^2}{\bar{C}}\nonumber \\
\bar{Q}_x(\bar{A},\bar{B},\bar{B}_1,\bar{C},\bar{D},\bar{E},\bar{F},\bar{G}) &= & \erf\lp\frac{(-2\bar{A}^2 \bar{E} - 2 \bar{A}\bar{B} + \bar{C}\bar{E})}{\sqrt{2}\sqrt{\bar{C}} \sqrt{\bar{C} - 2\bar{A}^2}}\rp -
\erf\lp\frac{(-2\bar{A}^2 \bar{D} - 2 \bar{A}\bar{B} + \bar{C}\bar{D})}{\sqrt{2}\sqrt{\bar{C}} \sqrt{\bar{C} - 2\bar{A}^2}}\rp
\nonumber \\
\bar{I}_{x1}(\bar{A},\bar{B},\bar{B}_1,\bar{C},\bar{D},\bar{E},\bar{F},\bar{G}) & = &  \frac{\sqrt{\bar{C}} e^{-\c_2\nu + \bar{G} + \bar{B}^2/(\bar{C} - 2 \bar{A}^2)} }{2 \sqrt{\bar{C} - 2 \bar{A}^2}}(2-\bar{Q}_x) \nonumber \\
\bar{I}_{x2}(\bar{D},\bar{E},\bar{F}) & = & e^{\frac{\bar{F}^2}{2}} \lp \frac{1}{2}\erfc\lp \frac{\bar{D}-\bar{F}}{\sqrt{2}}\rp - \frac{1}{2}\erfc\lp \frac{\bar{E}-\bar{F}}{\sqrt{2}}\rp\rp,
 \end{eqnarray}
 and after solving the remaining integrals obtain
\begin{eqnarray} \label{eq:negprac24a3a0a0}
  \mE_{{\mathcal U}_2}
{\cal Z}_1^{(2)}=    \bar{I}_{x1}(\bar{A},\bar{B},\bar{B}_1,\bar{C},\bar{D},\bar{E},\bar{F},\bar{G})+\bar{I}_{x2}(\bar{D},\bar{E},\bar{F}),
    \end{eqnarray}
and
\begin{eqnarray} \label{eq:negprac24a3a0}
\mE_{{\mathcal U}_3}\log\lp \mE_{{\mathcal U}_2}
{\cal Z}_1^{(2)}\rp=   \mE_{{\mathcal U}_3} \log\lp \bar{I}_{x1}(\bar{A},\bar{B},\bar{B}_1,\bar{C},\bar{D},\bar{E},\bar{F},\bar{G})+\bar{I}_{x2}(\bar{D},\bar{E},\bar{F})  \rp.
    \end{eqnarray}
Combining (\ref{eq:negprac24}), (\ref{eq:negprac24a3}), and (\ref{eq:negprac24a3a0}), one then finds
\begin{eqnarray}\label{eq:negprac25}
    \bar{\psi}_{rd}^{(2)}
 & = &    \frac{r_2^2}{2}
(1-\p_2\q_2)\c_2 -  \frac{\sigma^2}{2}
(1-\q_2)\c_2
- \nu_1 c_1
  -\gamma  r_1^2
     -\nu \beta
  \nonumber \\
  & & - \frac{1-\beta}{\c_2} \mE_{{\mathcal U}_3} \log\lp \bar{I}_{x1}(\bar{A},\bar{B},\bar{B}_1,\bar{C},\bar{D},\bar{E},\bar{F},\bar{G})+\bar{I}_{x2}(\bar{D},\bar{E},\bar{F})  \rp \nonumber \\
  & &
  - \frac{\beta}{\c_2} \mE_{{\mathcal U}_3} \log\lp \bar{I}_{x1}(\bar{A},\bar{B},\bar{B}_1,\bar{C},\bar{D},\bar{E},\bar{F},\bar{G})+\bar{I}_{x2}(\bar{D},\bar{E},\bar{F})  \rp
\nonumber \\
 & &  +  \gamma_{sq}r_2
+ \alpha r_2 \Bigg(\Bigg. \frac{1}{2\c_2r_2} \log \lp \frac{2\gamma_{sq}+\c_2r_2(1-\p_2)}{2\gamma_{sq}} \rp  +  \frac{\p_2}{2(2\gamma_{sq}+\c_2r_2(1-\p_2))}   \Bigg.\Bigg).
    \end{eqnarray}
The above is then sufficient to create the following system
\begin{eqnarray}\label{eq:dersystem1}
 \frac{d\bar{\psi}_{rd}^{(2)} }{d\p_2}
 =
 \frac{d\bar{\psi}_{rd}^{(2)} }{d\q_2}
 =
 \frac{d\bar{\psi}_{rd}^{(2)} }{d\c_2}
 =
 \frac{d\bar{\psi}_{rd}^{(2)} }{d\gamma_{sq}}
 =
 \frac{d\bar{\psi}_{rd}^{(2)} }{d\gamma}
 =
  \frac{d\bar{\psi}_{rd}^{(2)} }{d\nu_1}
  =
   \frac{d\bar{\psi}_{rd}^{(2)} }{d\nu}=0,
\end{eqnarray}
and obtain  $\hat{\p}_2$, $\hat{\q}_2$, $\hat{\c}_2$, $\hat{\gamma}_{sq}$, $\hat{\gamma}$, $\hat{\nu}_1$, and $\hat{\nu}$ as the solutions of the system.  Remarkable closed form relations between the optimal values of the parameters turn out to hold and provide a gigantic help in handling numerical evaluations.

\begin{corollary}
  \label{cor:closedformrel1}
  Assume the setup of Theorem \ref{thm:negthmprac1} (and implicitly Theorem\ref{thm:thmsflrdt1} and Corollary \ref{cor:cor1}). Set $r=2$. Then
\begin{eqnarray}\label{eq:clform1}
\hat{\gamma}_{sq} & = & \frac{1}{2}\frac{1-\hat{\p}_2}{1-\hat{\q}_2}\sqrt{\frac{\hat{\q}_2}{\hat{\p}_2}}\sqrt{\alpha} \nonumber \\
 \hat{\c}_2 & = & \frac{1}{r_2} \lp \frac{1}{1-\hat{\p}_2} \sqrt{\frac{\hat{\p}_2}{\hat{\q}_2}} - \frac{1}{1-\hat{\q}_2} \sqrt{\frac{\hat{\q}_2}{\hat{\p}_2}}\rp \sqrt{\alpha}.
\end{eqnarray}
\end{corollary}
\begin{proof}
Both equalities follow as analogues to (87) and (90) in \cite{Stojnicnegsphflrdt23} after line by line repetitions of their derivations with  minimal cosmetic sign and $r_2$ adjustments.
\end{proof}

 We present the concrete numerical values in Table \ref{tab:tab2}. The considered scenario is the same as in Table \ref{tab:tab1}, i.e., we again have $\alpha=0.057$, $\beta=0.01$, $c_1=0.5$, $r_1=1$, and $\sigma=\sqrt{0.00025}$. To enable a systematic view as to how the lifting machinery is progressing, we in parallel (show the results for the first level (1-sfl) from Table \ref{tab:tab1} and also for the partial second level (2-spf, obtained on the second level with $\hat{\p}_2=\hat{\q}_2=0$). As discussed in \cite{Stojnicflrdt23,Stojnicnegsphflrdt23,Stojnichopflrdt23}, the first level results correspond to what would be obtained utilizing the \emph{plain} RDT from, e.g., \cite{StojnicCSetam09,StojnicICASSP10var,StojnicUpper10,StojnicRegRndDlt10}, whereas the partial second level ones correspond to what would be obtained through the utilization of the \emph{partial} RDT from, e.g., \cite{StojnicLiftStrSec13}.
\begin{table}[h]
\caption{$2$-sfl RDT parameters; $\ell_0$ ML decoding ;  $\hat{\c}_1\rightarrow 1$; $n,\beta\rightarrow\infty$; $c_1=0.5$, $r_1=1$}\vspace{.1in}
\centering
\def\arraystretch{1.2}
{\small
\begin{tabular}{||l||c|c|c|c||c|c||c|c||c||c||}\hline\hline
 \hspace{-0in}$r$-sfl RDT                                             & $\hat{\gamma}_{sq}$   & $\hat{\gamma}$   & $\hat{\nu_1}$   & $\hat{\nu}$ &  $\hat{\p}_2$ & $\hat{\p}_1$     & $\hat{\q}_2$  & $\hat{\q}_1$ &  $\hat{\c}_2$    & $\xi_1^{(r)}(c_1,r_1)$  \\ \hline\hline
$1$-sfl RDT                                      & $0.1194$  & $0.1673$  & $-0.2711$ & $11.5934$ & $0$  & $\rightarrow 1$   & $0$ & $\rightarrow 1$
 &  $\rightarrow 0$  & \bl{$\mathbf{2.5149}$}
  \\ \hline\hline
$2$-spl RDT                                      & $0.0366$  & $0.3358$  & $-0.6871$ & $8.1128$ & $0$  & $\rightarrow 1$   & $0$ & $\rightarrow 1$
 &  $0.5913$  & \bl{$\mathbf{3.2935}$}
  \\ \hline\hline
$2$-sfl RDT                                      & $0.0406$  & $0.3580$  & $-0.7005$ & $8.4985$ & $0.5634$  & $\rightarrow 1$   & $0.2123$ & $\rightarrow 1$
 &  $0.7047$  & \bl{$\mathbf{3.6926}$}
  \\ \hline\hline
  \end{tabular}
  }
\label{tab:tab2}
\end{table}
Looking at the results from the table, a substantial improvement observed already on the second partial level is then further strengthened on the full second level.  Moreover, analogously to what we have done in the previous section when considering 1-sfl, we in Figure \ref{fig:fig3} supplement the 2-sfl $\xi_1(c_1,r_1)$ results from Table \ref{tab:tab2} obtained for a particular $c_1$ with the corresponding ones obtained for a much wider range of $c_1$.
  \begin{figure}[h]
\centering
\centerline{\includegraphics[width=1\linewidth]{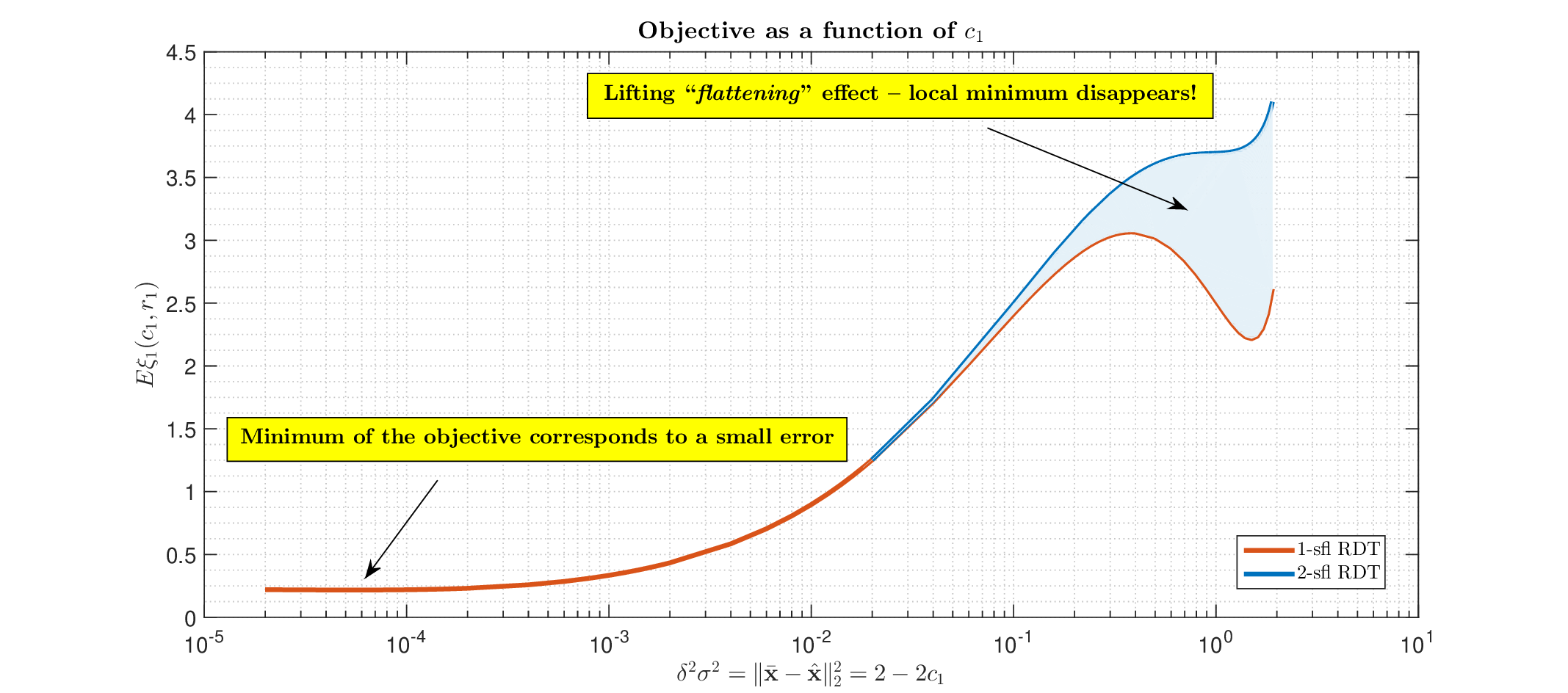}}
\caption{$\xi_1^{(2)}(c_1,r_1)$ as a function of $c_1$; $\alpha=0.057$, $\beta=0.01$, $\sigma=\sqrt{0.00025}$, and $r_1=1$.}
\label{fig:fig3}
\end{figure}
 From Figure \ref{fig:fig3} one observes that the objective curve is substantially lifted with the lift being particularly pronounced in the zone where the unwanted local minimum was present on the first level. Moreover, the lifting in that zone is further accompanied by a ``\emph{flattening}'' effect where the curve is flattened just enough that the local minimum completely disappears and instead becomes a \emph{point of infliction}. Not only does this make the overall $\ell_0$ ML decoding picture more logical (by effectively untangling the somewhat paradoxical conundrum from the first level), but it also has a very strong algorithmic consequences as well. Namely, assuming sufficiently large system's dimensions that ensure concentrations, any  $\ell_0$ norm preserving descending algorithm should actually converge to a global optimum and ultimately solve (\ref{eq:ex7}). This in turn, as discussed in \cite{Stojnicclupspreg20}, enables one to provide a recovery process with the best possible error within the ML context (see also e.g., \cite{WuV12a,WuV12b,WuV11} for analogous considerations under assumed prior's knowledge and corresponding Bayesian inference).
  \begin{figure}[h]
\centering
\centerline{\includegraphics[width=1\linewidth]{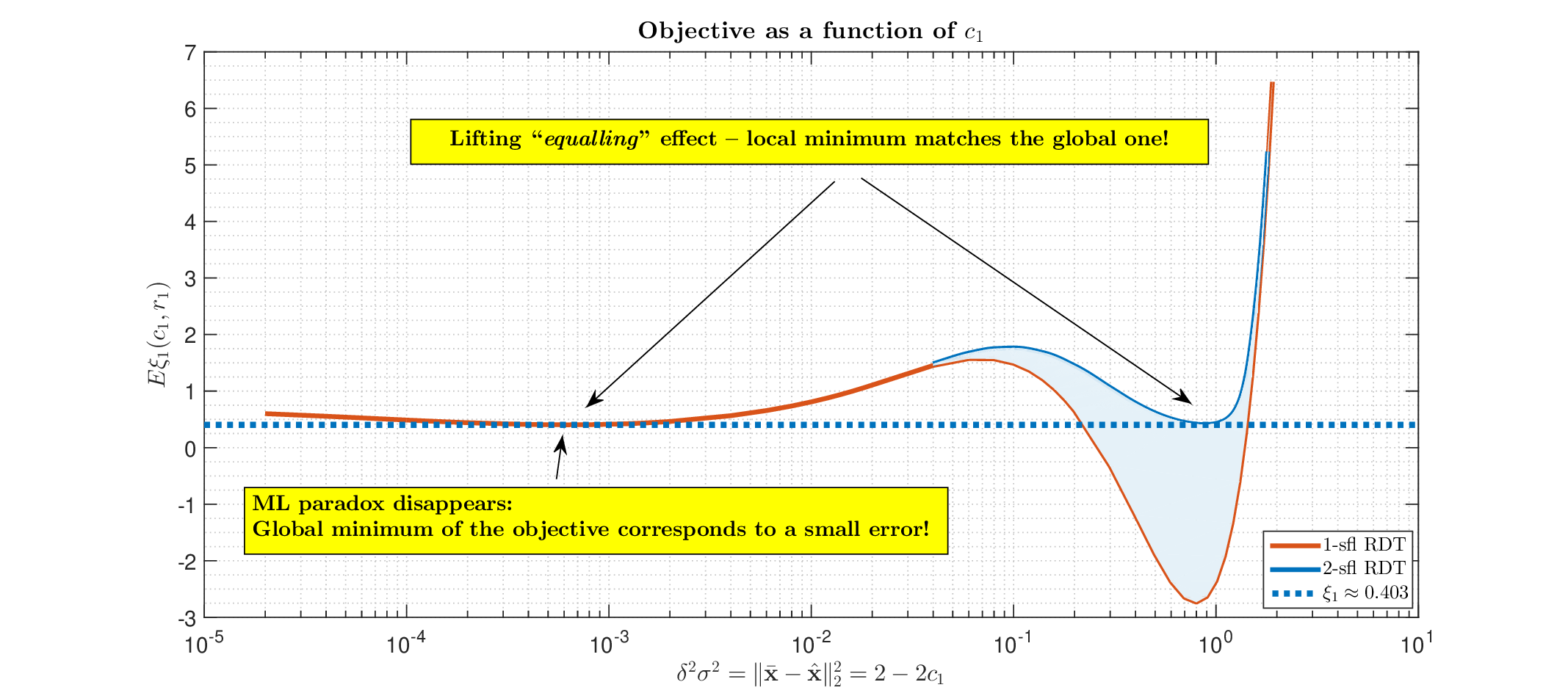}}
\caption{$\xi_1^{(2)}(c_1,r_1)$ as a function of $c_1$; $\alpha=0.563$, $\beta=0.4$, $\sigma=\sqrt{0.00025}$, and $r_1=1$.}
\label{fig:fig4}
\end{figure}

We complement the results from Figure \ref{fig:fig3} with the one shown in Figure \ref{fig:fig4} which highlight a different type of lifting phenomenon. In particular, Figure \ref{fig:fig4} shows the 2-sfl analogue results for the corresponding 1-sfl from Figure \ref{fig:fig2}. As in Figure \ref{fig:fig3}, the objective curve is again substantially lifted. This time though, one observes the lifting ``\emph{equalling}'' effect with the curve being lifted just enough that the unwanted global minimum (corresponding to a very large error) becomes a local one. While the notion that small objective should produce small error might be reaffirmed, the existence of the local minimum corresponding to a large recovery error is still unwanted. Two scenarios are possible: \textbf{\emph{(i)}} higher levels of lifting will eventually flatten the curve and local minimum will disappear; or \textbf{\emph{(ii)}} the large error paradox observed on the first level is in fact a natural structural property of the underlying problem that may or may not be present and its presence is directly related to the system parameters. As the results of the next section highlight, it is the second scenario that is of prevalent interest. In particular, the choice of $\alpha$ and $\beta$  (and $\sigma$) actually governs: \textbf{\emph{(1)}} whether or not multiple objective minima exist; and \textbf{\emph{(2)}} whether or not the unwanted ones (corresponding to a large recovery error) are actually global ones as well.
For example, this choice in Figure \ref{fig:fig4} is such that a favorable, small error, $\ell_0$ ML solution exists, but it might be unreachable by generic descending $\ell_0$ norm preserving  algorithms  ($d\ell_0$). Consequently, in such scenarios distinguishing the following three problem ``phases'' (determined by the system parameters $\alpha,\beta,$ and $\sigma$) positions itself as of critical algorithmic importance: \textbf{\emph{(i)}} $\ell_0$ ML recovery/decoding produces large error and is of no use; \textbf{\emph{(ii)}} the recovery error achievable by the $\ell_0$ ML is small but generically unachievable via $d\ell_0$; and \textbf{\emph{(iii)}} residual ML error is small and any descending $d\ell_0$ can achieve it. A similar phase-transitioning behavior was also observed in closely related recovery problems viewed through the prism of the Bayesian inference, see, e.g., \cite{KMSSZ12,KMSSZ12a,BarbierKMMZ18,ReevPfis2016,ReevesG12,ReevesG13}.

In the following sections we will discuss the algorithmic aspects in more details. Before getting there, we in the very next section discuss the third level of lifting. As mentioned above, what we have already presented contains all the main ingredients and is sufficient to proceed with the algorithmic discussion. It is, however, useful to first double check if any further substantial improvements are possible on the third level of lifting.

\subsubsection{$r=3$ -- third level of lifting}
\label{sec:thirdlev}

For $r=3$, we have that $\hat{\p}_1\rightarrow 1$ and $\hat{\q}_1\rightarrow 1$  as well as  $\hat{\p}_{r+1}=\hat{\p}_{4}=\hat{\q}_{r+1}=\hat{\q}_{4}=0$. After setting
\begin{eqnarray}\label{eq:3negpraccalz}
{\cal Z}_0^{(3)} & = & e^{-\c_2   \min \lp- \frac{\lp \sqrt{1-\q_2}\h_{i}^{(2)} + \sqrt{\q_2-\q_3}\h_{i}^{(3)}+ \sqrt{\q_3}\h_{i}^{(4)}
     \rp^2}{4\gamma}+\nu,0\rp } \nonumber \\
{\cal Z}_1^{(3)}& = &  e^{\c_2 \lp \lp \sqrt{1-\q_2}\h_{i}^{(2)} + \sqrt{\q_2-\q_3}\h_{i}^{(3)}+ \sqrt{\q_3}\h_{i}^{(4)}\rp\bar{\x}_i\sqrt{n} -  \min \lp- \frac{\lp \sqrt{1-\q_2}\h_{i}^{(2)} + \sqrt{\q_2-\q_3}\h_{i}^{(3)}
 + \sqrt{\q_3}\h_{i}^{(4)}
  +\nu_1\bar{\x}\sqrt{n}   \rp^2}{4\gamma}+\nu,0\rp \rp },\nonumber \\
\end{eqnarray}
we further write analogously to (\ref{eq:negprac24}),
\begin{align}\label{eq:3negprac24}
    \bar{\psi}_{rd}^{(3)}    & =   \frac{r_2^2}{2}
(1-\p_2\q_2)\c_2+ \frac{r_2^2}{2}
(\p_2\q_2-\p_3\q_3)\c_3
- \frac{\sigma^2}{2}
(1-\q_2)\c_2+ \frac{\sigma^2}{2}
(\q_2-\q_3)\c_3
- \nu_1 c_1
  -\gamma  r_1^2
     -\nu \beta
  \nonumber \\
& \quad - \frac{1-\beta}{\c_3}\mE_{{\mathcal U}_4}\log  \lp \mE_{{\mathcal U}_3} \lp \mE_{{\mathcal U}_2}
{\cal Z}_0^{(3)}\rp^{\frac{\c_3}{\c_2}}  \rp
 - \frac{1}{n\c_3} \sum_{i=1}^{k}\mE_{{\mathcal U}_4}\log   \lp \mE_{{\mathcal U}_3} \lp \mE_{{\mathcal U}_2}
{\cal Z}_1^{(3)}
\rp^{\frac{\c_3}{\c_2}} \rp
\nonumber \\
& \quad  +  \gamma_{sq}r_2 - \frac{\alpha}{\c_3}\mE_{{\mathcal U}_4}\log\lp \mE_{{\mathcal U}_3} \lp \mE_{{\mathcal U}_2} e^{-\c_2r_2\frac{\lp\sqrt{1-\p_2}\u_1^{(2,2)} +\sqrt{\p_2-\p_3}\u_1^{(2,3)}+\sqrt{\p_3}\u_1^{(2,4)} \rp^2}{4 \gamma_{sq}}}\rp^{\frac{\c_3}{\c_2}}\rp \nonumber \\
& =   \frac{r_2^2}{2}
(1-\p_2\q_2)\c_2+ \frac{r_2^2}{2}
(\p_2\q_2-\p_3\q_3)\c_3
- \frac{\sigma^2}{2}
(1-\q_2)\c_2+ \frac{\sigma^2}{2}
(\q_2-\q_3)\c_3
- \nu_1 c_1
  -\gamma  r_1^2
     -\nu \beta
  \nonumber \\
& \quad - \frac{1-\beta}{\c_3}\mE_{{\mathcal U}_4}\log  \lp \mE_{{\mathcal U}_3} \lp \mE_{{\mathcal U}_2}
{\cal Z}_0^{(3)}
     \rp^{\frac{\c_3}{\c_2}}  \rp
  - \frac{1}{n\c_3} \sum_{i=1}^{k}\mE_{{\mathcal U}_4}\log   \lp \mE_{{\mathcal U}_3} \lp \mE_{{\mathcal U}_2}
{\cal Z}_1^{(3)}
\rp^{\frac{\c_3}{\c_2}} \rp
 \nonumber \\
  & \quad
 +\gamma_{sq}r_2  +\alpha r_2\Bigg(\Bigg. \frac{1}{2\c_2 r_2} \log \lp \frac{2\gamma_{sq}+\c_2 r_2(1-\p_2)}{2\gamma_{sq}} \rp
   \nonumber \\
& \quad
  +\frac{1}{2\c_3 r_2} \log \lp \frac{2\gamma_{sq}+\c_2 r_2(1-\p_2)+\c_3 r_2(\p_2-\p_3)}{2\gamma_{sq}+\c_2 r_2(1-\p_2)} \rp  +  \frac{\p_3}{2(2\gamma_{sq}+\c_2 r_2(1-\p_2)+\c_3 r_2(\p_2-\p_3))}   \Bigg.\Bigg),
    \end{align}
where the last term on the right hand side of the first equality is handled as the corresponding quantity in \cite{Stojnicnegsphflrdt23}. Recalling again on $\bar{\x}_i=\frac{1}{\sqrt{k}}$,  we first set
\begin{eqnarray}\label{eq:3negprac24a0a0}
 D^{(3)} & = & \frac{(-\h_1^{(3)}\sqrt{\q_2-\q_3} - \h_1^{(4)}\sqrt{\q_3}-2\sqrt{\gamma\nu})}{\sqrt{1-\q_2}} \nonumber \\
 E^{(3)} & = & \frac{(-\h_1^{(3)}\sqrt{\q_2-\q_3} - \h_1^{(4)}\sqrt{\q_3} +2\sqrt{\gamma\nu})}{\sqrt{1-\q_2}} \nonumber \\
A^{(3)} & = & \sqrt{\c_2(1-\q_2)} \nonumber \\
B^{(3)} & = & \h_1^{(3)}\sqrt{\c_2(\q_2-\q_3)} + \h_1^{(4)}\sqrt{\c_2\q_3}\nonumber \\
C^{(3)} & = & 4\gamma,
 \end{eqnarray}
and after solving the remaining integrals obtain
\begin{eqnarray} \label{eq:3negprac24a3a0}
  \mE_{{\mathcal U}_2}
{\cal Z}_0^{(2)} =     I_{x1}(A^{(3)},B^{(3)},C^{(3)},D^{(3)},E^{(3)})+I_{x2}(D^{(3)},E^{(3)}),
    \end{eqnarray}
and
\begin{eqnarray} \label{eq:3negprac24a3}
\mE_{{\mathcal U}_4}\log \lp  \mE_{{\mathcal U}_3} \lp \mE_{{\mathcal U}_2}
{\cal Z}_0^{(2)}\rp^{\frac{\c_3}{\c_2}} \rp
=
 \mE_{{\mathcal U}_4} \log \lp \mE_{{\mathcal U}_3} \lp   I_{x1}(A^{(3)},B^{(3)},C^{(3)},D^{(3)},E^{(3)})+I_{x2}(D^{(3)},E^{(3)}) \rp^{\frac{\c_3}{\c_2}} \rp,
    \end{eqnarray}
with $I_{x1}(\cdot)$ and $I_{x2}(\cdot)$  as in (\ref{eq:negprac24a0a0}). We continue as in the previous section and  analogously  set
\begin{eqnarray}\label{eq:3negprac24a0a0a0}
 \bar{D}^{(3)} & = & \frac{(-(\h_1^{(3)}\sqrt{\q_2-\q_3}+\h_1^{(4)}\sqrt{\q_3}+\bar{d}\nu_1) -2\sqrt{\gamma\nu})}{\sqrt{1-\q_2}} \nonumber \\
 \bar{E}^{(3)} & = & \frac{(-(\h_1^{(3)}\sqrt{\q_2-\q_3}+\h_1^{(4)}\sqrt{\q_3}+\bar{d}\nu_1) +2\sqrt{\gamma\nu})}{\sqrt{1-\q_2}} \nonumber \\
\bar{A}^{(3)} & = & \sqrt{\c_2(1-\q_2)} \nonumber \\
\bar{B}_1^{(3)} & = & \h_1^{(3)}\sqrt{\c_2(\q_2-\q_3)} +  \h_1^{(4)}\sqrt{\c_2\q_3} + \sqrt{\c_2}\bar{d}\nu_1\nonumber \\
\bar{C}^{(3)} & = & 4\gamma \nonumber \\
\bar{F}^{(3)} & = & \c_2\bar{d}\sqrt{1-\q_2} \nonumber \\
\bar{B}^{(3)} & = & \frac{\bar{F}^{(3)}\bar{C}^{(3)}}{2\bar{A}^{(3)}}+\bar{B}_1^{(3)} \nonumber \\
\bar{G}^{(3)} & = &\frac{\lp \bar{B}_1^{(3)} \rp^2-\lp\bar{B}^{(3)} \rp^2}{\bar{C}^{(3)}},
 \end{eqnarray}
 and after solving the remaining integrals obtain
\begin{eqnarray} \label{eq:3negprac24a3a0a0}
  \mE_{{\mathcal U}_2}
{\cal Z}_1^{(2)}=    \bar{I}_{x1}(\bar{A}^{(3)},\bar{B}^{(3)},\bar{B}_1^{(3)},\bar{C}^{(3)},\bar{D}^{(3)},\bar{E}^{(3)},\bar{F}^{(3)},\bar{G})^{(3)}+\bar{I}_{x2}(\bar{D}^{(3)},\bar{E}^{(3)},\bar{F}^{(3)}),
    \end{eqnarray}
and
\begin{multline} \label{eq:3negprac24a3a0}
\mE_{{\mathcal U}_4}\log  \lp  \mE_{{\mathcal U}_3} \lp \mE_{{\mathcal U}_2}
{\cal Z}_1^{(2)}\rp^{\frac{\c_3}{\c_2}} \rp
=  \\
  \mE_{{\mathcal U}_4}\log  \lp  \mE_{{\mathcal U}_3} \lp \bar{I}_{x1}(\bar{A}^{(3)},\bar{B}^{(3)},\bar{B}_1^{(3)},\bar{C}^{(3)},\bar{D}^{(3)},\bar{E}^{(3)},\bar{F}^{(3)},\bar{G})^{(3)}+\bar{I}_{x2}(\bar{D}^{(3)},\bar{E}^{(3)},\bar{F}^{(3)})\rp^{\frac{\c_3}{\c_2}} \rp,
    \end{multline}
with $\bar{I}_{x1}(\cdot)$ and $\bar{I}_{x2}(\cdot)$  as in (\ref{eq:negprac24a0a0a0}). Combining (\ref{eq:3negprac24}), (\ref{eq:3negprac24a3}), and (\ref{eq:3negprac24a3a0}), one then finds
\begin{align}\label{eq:3negprac25}
    \bar{\psi}_{rd}^{(3)}
& =   \frac{r_2^2}{2}
(1-\p_2\q_2)\c_2+ \frac{r_2^2}{2}
(\p_2\q_2-\p_3\q_3)\c_3
- \frac{\sigma^2}{2}
(1-\q_2)\c_2+ \frac{\sigma^2}{2}
(\q_2-\q_3)\c_3
- \nu_1 c_1
  -\gamma  r_1^2
     -\nu \beta
  \nonumber \\
& \quad - \frac{1-\beta}{\c_3}
 \mE_{{\mathcal U}_4} \log \lp \mE_{{\mathcal U}_3} \lp   I_{x1}(A^{(3)},B^{(3)},C^{(3)},D^{(3)},E^{(3)})+I_{x2}(D^{(3)},E^{(3)}) \rp^{\frac{\c_3}{\c_2}} \rp     \nonumber \\
     & \quad
  - \frac{\beta}{\c_3}
 \mE_{{\mathcal U}_4}\log  \lp  \mE_{{\mathcal U}_3} \lp \bar{I}_{x1}(\bar{A}^{(3)},\bar{B}^{(3)},\bar{B}_1^{(3)},\bar{C}^{(3)},\bar{D}^{(3)},\bar{E}^{(3)},\bar{F}^{(3)},\bar{G})^{(3)}+\bar{I}_{x2}(\bar{D}^{(3)},\bar{E}^{(3)},\bar{F}^{(3)})\rp^{\frac{\c_3}{\c_2}} \rp \nonumber \\
  & \quad
 +\gamma_{sq}r_2  +\alpha r_2\Bigg(\Bigg. \frac{1}{2\c_2 r_2} \log \lp \frac{2\gamma_{sq}+\c_2 r_2(1-\p_2)}{2\gamma_{sq}} \rp
   \nonumber \\
& \quad
  +\frac{1}{2\c_3 r_2} \log \lp \frac{2\gamma_{sq}+\c_2 r_2(1-\p_2)+\c_3 r_2(\p_2-\p_3)}{2\gamma_{sq}+\c_2 r_2(1-\p_2)} \rp +  \frac{\p_3}{2(2\gamma_{sq}+\c_2 r_2(1-\p_2)+\c_3 r_2(\p_2-\p_3))}   \Bigg.\Bigg).
    \end{align}
Following further the methodology of the previous sections, one observes that the above is sufficient to create the following system
\begin{eqnarray}\label{eq:dersystem1}
 \frac{d\bar{\psi}_{rd}^{(3)} }{d\p_2}
 =
 \frac{d\bar{\psi}_{rd}^{(3)} }{d\q_2}
 =
 \frac{d\bar{\psi}_{rd}^{(3)} }{d\c_2}
 =
 \frac{d\bar{\psi}_{rd}^{(3)} }{d\p_3}
 =
 \frac{d\bar{\psi}_{rd}^{(3)} }{d\q_3}
 =
 \frac{d\bar{\psi}_{rd}^{(3)} }{d\c_3}
 =
 \frac{d\bar{\psi}_{rd}^{(3)} }{d\gamma_{sq}}
 =
 \frac{d\bar{\psi}_{rd}^{(3)} }{d\gamma}
 =
  \frac{d\bar{\psi}_{rd}^{(3)} }{d\nu_1}
  =
   \frac{d\bar{\psi}_{rd}^{(3)} }{d\nu}=0,
\end{eqnarray}
and obtain  $\hat{\p}_2$, $\hat{\q}_2$, $\hat{\c}_2$, $\hat{\p}_3$, $\hat{\q}_3$, $\hat{\c}_3$, $\hat{\gamma}_{sq}$, $\hat{\gamma}$, $\hat{\nu}_1$, and $\hat{\nu}$ as the solutions of the system.  Moreover, it turns out that the optimal values of the parameters are connected via closed form relations similar to the ones obtained on the second level. Utilizing again the derivations of \cite{Stojnicnegsphflrdt23} and minimally adjusting for the sign and $r_2$, one has the following corollary.
\begin{corollary}
  \label{cor:3closedformrel1}
  Assume the setup of Theorem \ref{thm:negthmprac1} (and implicitly Theorem\ref{thm:thmsflrdt1} and Corollary \ref{cor:cor1}). Set $r=2$. Then
\begin{eqnarray}\label{eq:3clform1}
\hat{\gamma}_{sq} & = & \frac{1}{2}\frac{1-\hat{\p}_2}{1-\hat{\q}_2}\frac{\hat{\q}_2-\hat{\q}_3}{\hat{\p}_2-\hat{\p}_3}\sqrt{\frac{\hat{\p}_3}{\hat{\q}_3}}\sqrt{\alpha} \nonumber \\
 \hat{\c}_2 & = & \frac{1}{r_2} \lp
 \frac{1}{1-\hat{\p}_2} \frac{\hat{\p}_2-\hat{\p}_3}{\hat{\q}_2-\hat{\q}_3} \sqrt{\frac{\hat{\q}_3}{\hat{\p}_3}}
 -
 \frac{1}{1-\hat{\q}_2} \frac{\hat{\q}_2-\hat{\q}_3}{\hat{\p}_2-\hat{\p}_3} \sqrt{\frac{\hat{\p}_3}{\hat{\q}_3}}
    \rp \sqrt{\alpha} \nonumber \\
 \hat{\c}_3 & = & \frac{1}{r_2} \lp
\frac{1}{\hat{\p}_2-\hat{\p}_3} \sqrt{\frac{\hat{\p}_3}{\hat{\q}_3}}
-
\frac{1}{\hat{\q}_2-\hat{\q}_3} \sqrt{\frac{\hat{\q}_3}{\hat{\p}_3}}
    \rp \sqrt{\alpha}.
\end{eqnarray}
\end{corollary}
\begin{proof}
 All three equalities follow as analogues to (142), (145), and (147)  in \cite{Stojnicnegsphflrdt23} after line by line repetitions of their derivations.
\end{proof}

The concrete numerical values obtained based on the above are shown in Table \ref{tab:tab3}. We maintain parallelism with Tables \ref{tab:tab1} and  \ref{tab:tab2} and again have $\alpha=0.057$, $\beta=0.01$, $c_1=0.5$, $r_1=1$, and $\sigma=\sqrt{0.00025}$.
\begin{table}[h]
\caption{$3$-sfl RDT parameters; $\ell_0$ ML decoding ;  $\hat{\c}_1\rightarrow 1$; $n,\beta\rightarrow\infty$; $c_1=0.5$, $r_1=1$}\vspace{.1in}
\centering
\def\arraystretch{1.2}
{\footnotesize
\begin{tabular}{||l||c|c|c|c||c|c|c||c|c|c||c|c||c||}\hline\hline
 \hspace{-0in}$r$-sfl                                             & $\hat{\gamma}_{sq}$   & $\hat{\gamma}$   & $\hat{\nu_1}$   & $\hat{\nu}$ &  $\hat{\p}_3$ &  $\hat{\p}_2$ & $\hat{\p}_1$     & $\hat{\q}_3$  & $\hat{\q}_2$  & $\hat{\q}_1$ &  $\hat{\c}_3$  &  $\hat{\c}_2$    & $\xi_1^{(r)}$  \\ \hline\hline
$1$-sfl                                      & $0.119$  & $0.167$  & $-0.271$ & $11.593$ & $0$  & $0$  & $\rightarrow 1$  & $0$   & $0$ & $\rightarrow 1$
 & $0$  &  $\rightarrow 0$  & \bl{$\mathbf{2.5149}$}
  \\ \hline\hline
$2$-spl                                        & $0.036$  & $0.335$  & $-0.687$ & $8.112$  & $0$  & $0$  & $\rightarrow 1$  & $0$   & $0$ & $\rightarrow 1$
 & $0$  &  $0.591$  & \bl{$\mathbf{3.2935}$}
  \\ \hline\hline
$2$-sfl                                        & $0.040$  & $0.358$  & $-0.700$ & $8.498$    & $0$  & $0.563$  & $\rightarrow 1$  & $0$   & $0.212$ & $\rightarrow 1$
 & $0$  &  $0.704$  & \bl{$\mathbf{3.6926}$}
  \\ \hline\hline
$3$-sfl                                        & $0.037$  & $0.389$  & $-0.763$ & $8.506$   & $0.557$   & $0.988$  & $\rightarrow 1$    & $0.210$ & $0.904$  & $\rightarrow 1$
 &  $0.689$  & $1.452$  & \bl{$\mathbf{3.6956}$}
  \\ \hline\hline
  \end{tabular}
  }
\label{tab:tab3}
\end{table}
This time though not much of an improvement is observed  compared to the first and second level.  Analogously to what we have done in the previous two sections, we in Figure \ref{fig:fig5} supplement the 3-sfl $\xi_1(c_1,r_1)$ results from Table \ref{tab:tab3} obtained for a particular $c_1$ with the corresponding ones obtained for a much wider range of $c_1$.
  \begin{figure}[h]
\centering
\centerline{\includegraphics[width=1\linewidth]{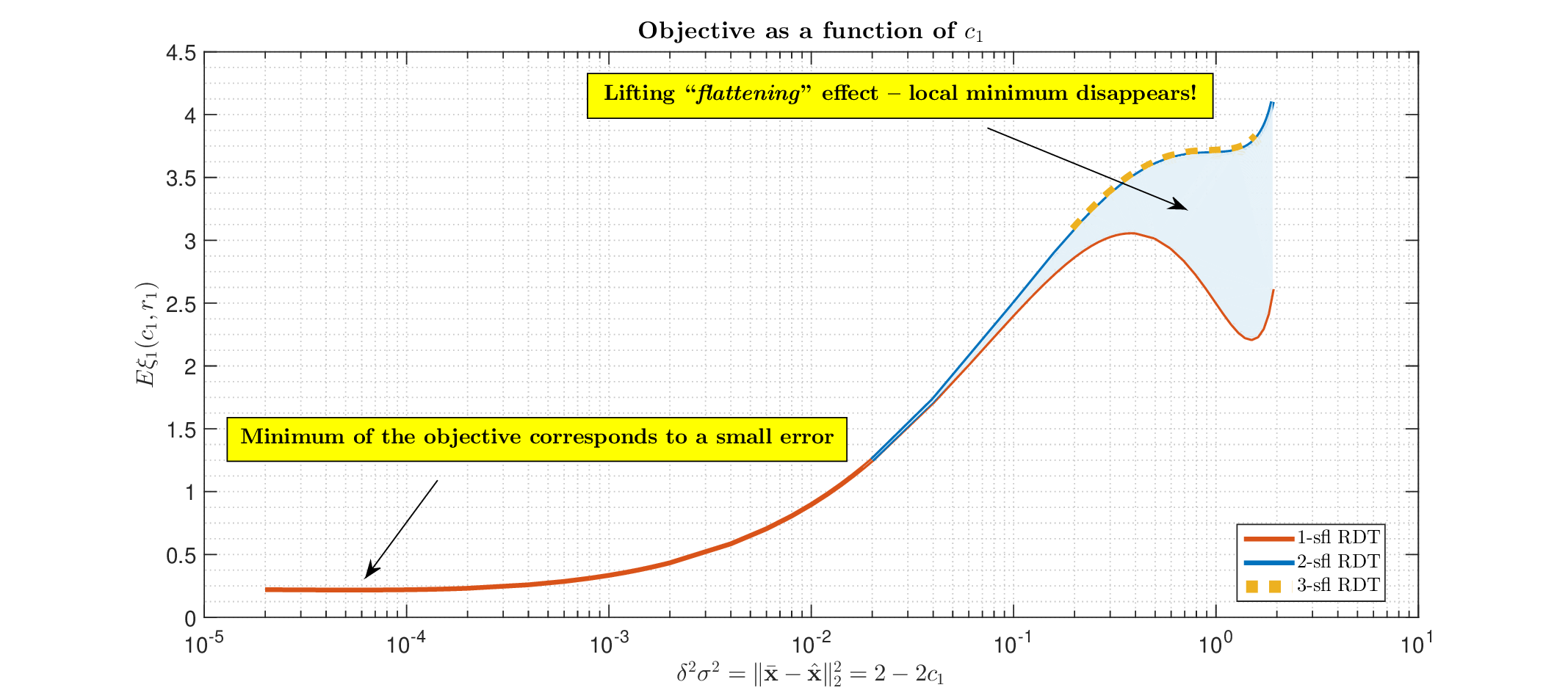}}
\caption{$\xi_1^{(2)}(c_1,r_1)$ as a function of $c_1$; $\alpha=0.057$, $\beta=0.01$, $\sigma=\sqrt{0.00025}$, and $r_1=1$.}
\label{fig:fig5}
\end{figure}
 Figure \ref{fig:fig5} shows that the trend from Table \ref{tab:tab3} continues throughout the entire $c_1$ range with  not much of an improvement being observed  (it is usually $\sim 0.1\%$ or smaller). While this heavily relies on substantial numerical evaluations (which necessarily bring imprecisions), it still suggests that, for all practical purposes, considerations on the first two levels are likely to be sufficient.
\begin{figure}[h]
\centering
\centerline{\includegraphics[width=1\linewidth]{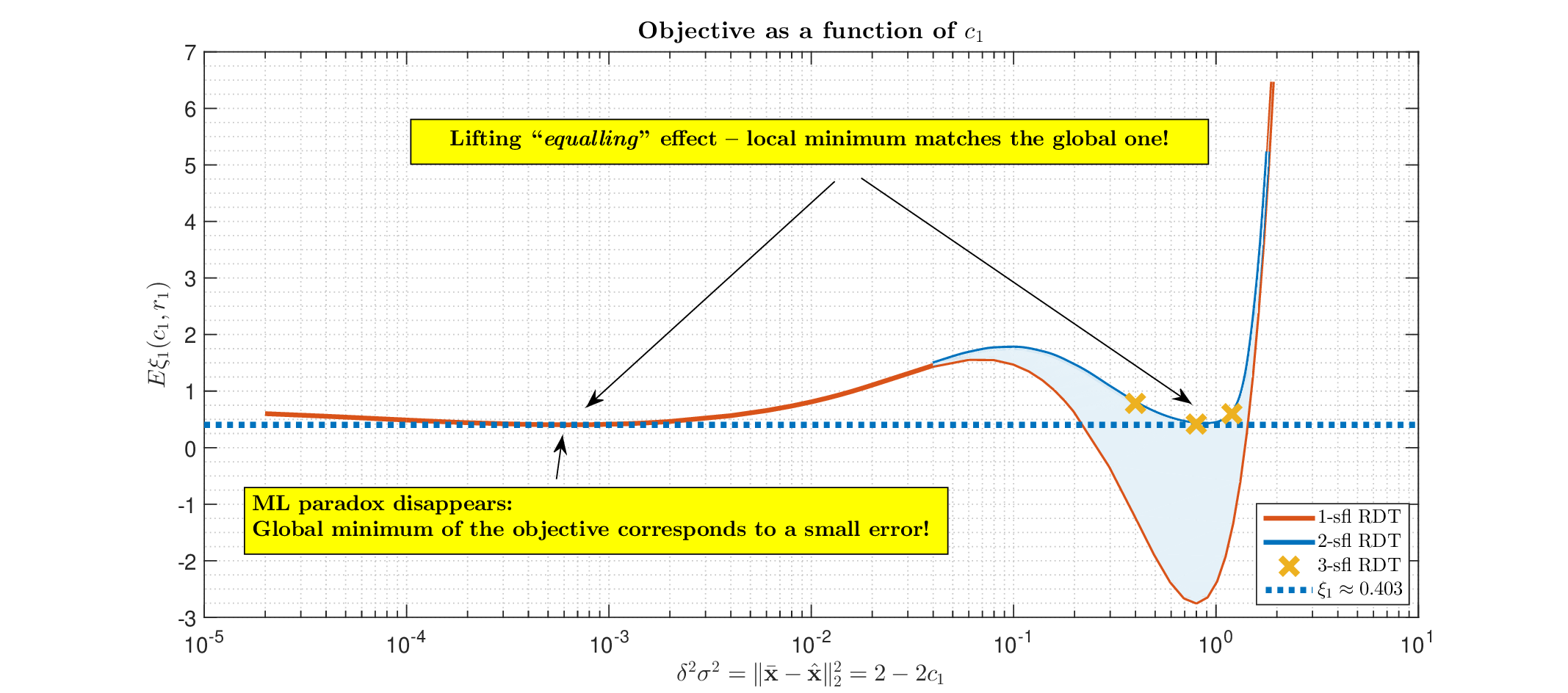}}
\caption{$\xi_1^{(2)}(c_1,r_1)$ as a function of $c_1$; $\alpha=0.563$, $\beta=0.4$, $\sigma=\sqrt{0.00025}$, and $r_1=1$.}
\label{fig:fig6}
\end{figure}
Figure \ref{fig:fig6} shows a very similar behavior for $\alpha=0.563$, $\beta=0.4$ regime as well.

\subsection{Algorithmic implications}
\label{sec:algimp}

One can utilize the presented results to precisely characterize the range of system parameters $\alpha$, $\beta$, and $\sigma$ for which: \textbf{\emph{(i)}} the $\textbf{RMSE}$ of the $\ell_0$ ML decoding is small; and \textbf{\emph{(ii)}} such small $\textbf{RMSE}$ is achievable/unachievable by any $d\ell_0$. In, particular, one considers the objective $\xi_1$ and views it as a function of $c_1$ (or implicitly $\delta=\mE(\textbf{RMSE})=\frac{\mE \|\bar{\x}-\hat{\x}\|_2}{\sigma}\rightarrow\frac{\sqrt{2-2c_1}}{\sigma}$). If $\alpha,$ $\beta$, and $\sigma$ are such that  the objective  has single local minimum, then the $\ell_0$ ML decoding is solvable by \emph{any} $d\ell_0$. If it has multiple local minima then $d\ell_0$'s generically fail to optimally solve $\ell_0$ ML decoding. In Figure \ref{fig:fig7}, we show $(\alpha,\beta)$ regions obtained utilizing the above presented machinery for $\sigma=\sqrt{0.00025}$ (as the above discussions suggested that the improvements on the higher lifting levels seem unsubstantial, we focused on implementing 2-sfl to make numerical evaluations a bit simpler). As can be seen from the figure, there are two phase transition (PT) curves: \textbf{\emph{(i)}} the lower one, denoted by $d\ell_0$ PT (\emph{descending} $\ell_0$ PT), and \textbf{\emph{(ii)}} the higher one, denoted by $a\ell_0$ PT (\emph{any} $\ell_0$ PT).  These curves delimit three regions in $(\alpha,\beta)$ space with the following features: \textbf{\emph{(1)}} in the region to the right of the $d\ell_0$ PT, the residual $\textbf{RMSE}$ of the $\ell_0$ ML decoding is small (i.e., comparable to $\sigma$) and  can be achieved by any $d\ell_0$ (we say that it is $d\ell_0$ achievable); \textbf{\emph{(2)}} in the region to the left of the $d\ell_0$ PT but to the right of the $a\ell_0$ PT, the residual $\textbf{RMSE}$ is small but it can not be achieved by a generic $d\ell_0$ (we say it is $d\ell_0$ unachievable); and \textbf{\emph{(3)}} in the region to the left of the $a\ell_0$ PT the residual $\textbf{RMSE}$ is large and $\ell_0$ ML decoding is of not much practical use. For the completeness, we mention that one can further split the third region into two subregions -- one where the objective has multiple local minima and the other where it has a single global minimum. Since in either of these cases the error is large (and the whole $\ell_0$ ML decoding concept basically useless), we skip introducing such a differentiation. We just note that in either case $d\ell_0$ actually typically finds the global minimum.

In addition to the $d\ell_0$, and $a\ell_0$ PTs, we also plot the standard $\ell_1$ PT as a general comparing benchmark.  As can be seen from the figure, $\ell_0$ ML decoding is way more powerful than the $\ell_1$. Interestingly, one observes that even $d\ell_0$ can outperform $\ell_1$ in highly under-sampled systems (small $\alpha$). This in fact is very surprising as we are unaware of any algorithmic structure that can beat $\ell_1$ in such a plain setup (unchangeable $A$ and a priori unknown deterministic $\bar{\x}$).  Even more surprising, in an almost full system ($\alpha$ close to $1$), $d\ell_0$  can even achieve  sparsity/under-sampling ratio of $1$.
\begin{figure}[h]
\centering
\centerline{\includegraphics[width=1\linewidth]{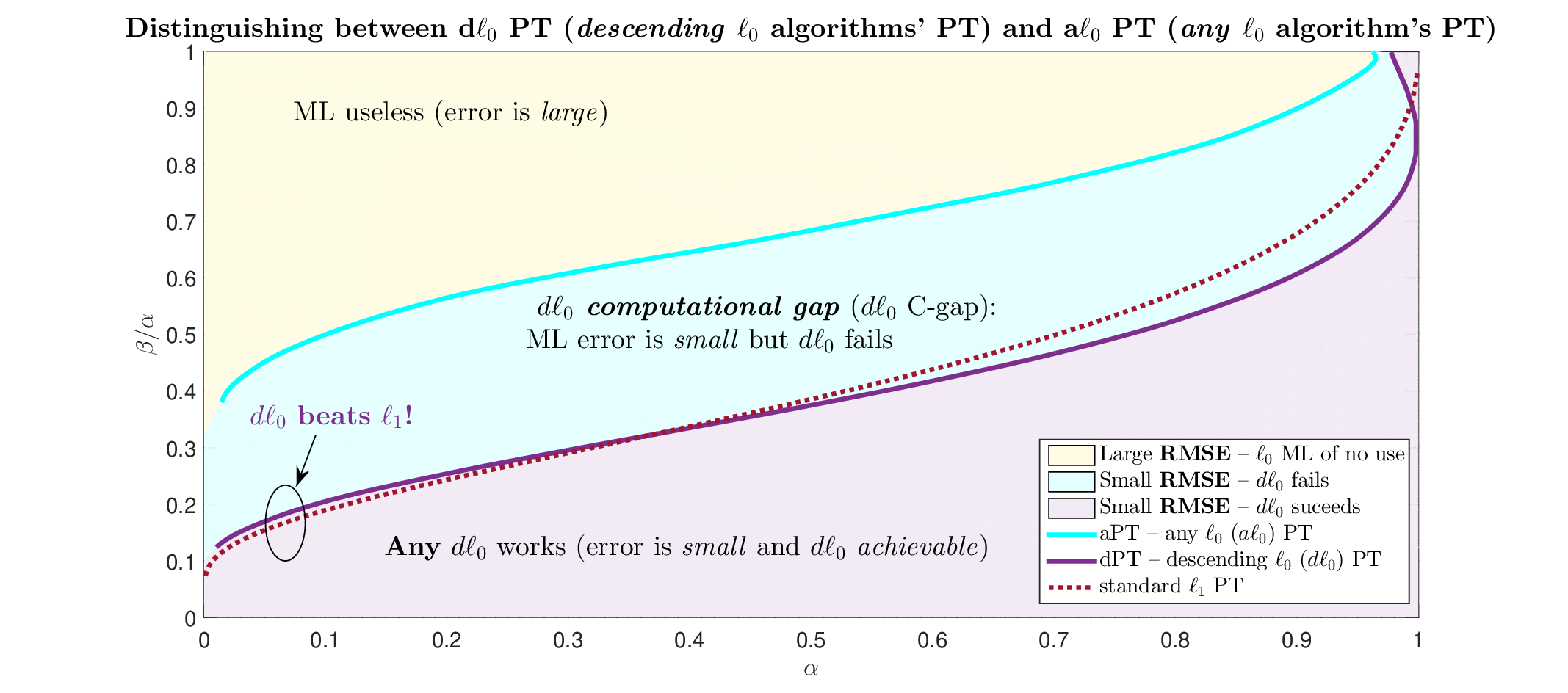}}
\caption{aPT and dPT; binary $\bar{\x}$, $\sigma=\sqrt{0.00025}$, and $r_1=1$.}
\label{fig:fig7}
\end{figure}

\subsection{Different types of signal}
\label{sec:diffsig}

To be able to obtain concrete numerical values, we in the above considerations selected $\bar{\x}_i=\frac{1}{\sqrt{k}},1\leq i\leq k$. Such a choice corresponds to a so-called binary scenario which is typically viewed  as the hardest one for generic algorithms (those that have no prior knowledge about the structure of the signal). We now supplement those results by considering an example of a statistical signal closely related to the so-called Gauss-Bernoulli signals  typically considered throughout the literature. Namely, we consider the so-called Gaussian scenario, where $\bar{\x}_i=\frac{{\cal N}(0,1)}{\sqrt{k}},1\leq i\leq k$ and all $\bar{\x}_i$ are independent of each other. It is trivial to observe that in a large dimensional contexts this makes $\|\bar{\x}\|_2\rightarrow 1$ with probability 1. Moreover, the entire machinery presented in earlier sections remains in place even when the signals are statistical. The only additional thing is that now $\bar{d}$ is a standard normal and one needs to eventually integrate over it as well. This does make the numerical evaluations a bit more challenging. Nonetheless, we in figures below show the obtained results for some interesting scenarios.

We start with Figure \ref{fig:fig8} where we take $\alpha=0.59$ and $\beta=0.4$ and keep $\sigma=\sqrt{0.00025}$. As the flattening effect is already in full power for $\beta=0.4$, one sees a dramatic improvement over the binary scenario. In particular, for the same under-sampling $\alpha$ and binary $\bar{\x}$, Figure \ref{fig:fig7} gives  $\beta\approx 0.244$. On the other hand, it is interesting to note that the results are very much comparable to the ones from \cite{KMSSZ12,KMSSZ12a} where for a similar Gaussian scenario and $\beta=0.4$ it was obtained that $\alpha\approx 0.59$. This happens despite a few key conceptual differences discussed earlier -- statistical identicalness of all $n$ $\bar{\x}$'s coordinates is now not in place and \cite{KMSSZ12,KMSSZ12a} consider Bayesian inference context (which is optimal and feasible if the true and postulated priors are identical and (for the algorithmic relevance) known beforehand). Moreover, the lack of a more profound impact of a priori available signal knowledge is also in an agreement with was hinted on in \cite{ReevesG12} while relying on the ML distortion bounds. It should also be noted that in \cite{ONOOK18} a MAP decoding similar to the one from (\ref{eq:ex7}) is considered. An entropy function, viewed as a histogram over all sparsity levels, is analyzed via replica methods and the obtained results in the Gaussian scenario are very similar to the ones from \cite{KMSSZ12,KMSSZ12a}. Furthermore, as Figure 5 in \cite{ONOOK18} and Figure 3 in \cite{BarbierLSKZ24}) show, the results from \cite{ONOOK18} fairly closely match the ones obtained for the very same noiseless Gauss-Bernoulli scenario in \cite{BarbierLSKZ23,BarbierLSKZ24}  relying on replica analysis and its connections to AMP and particularly ASP (approximate survey propagation) methods considered earlier in slightly different contexts in \cite{AKUZ19,LucibelloSL19}. As these figures also show, the impact of fully available signals prior needed in the optimal Bayes context is rather limited which is in a perfect agreement with what was hinted on in \cite{ReevesG12} and  discussed above.

Given a strong improvement over the binary case, it is of interest to shed a bit more light on this effect.
 \begin{figure}[h]
\centering
\centerline{\includegraphics[width=1\linewidth]{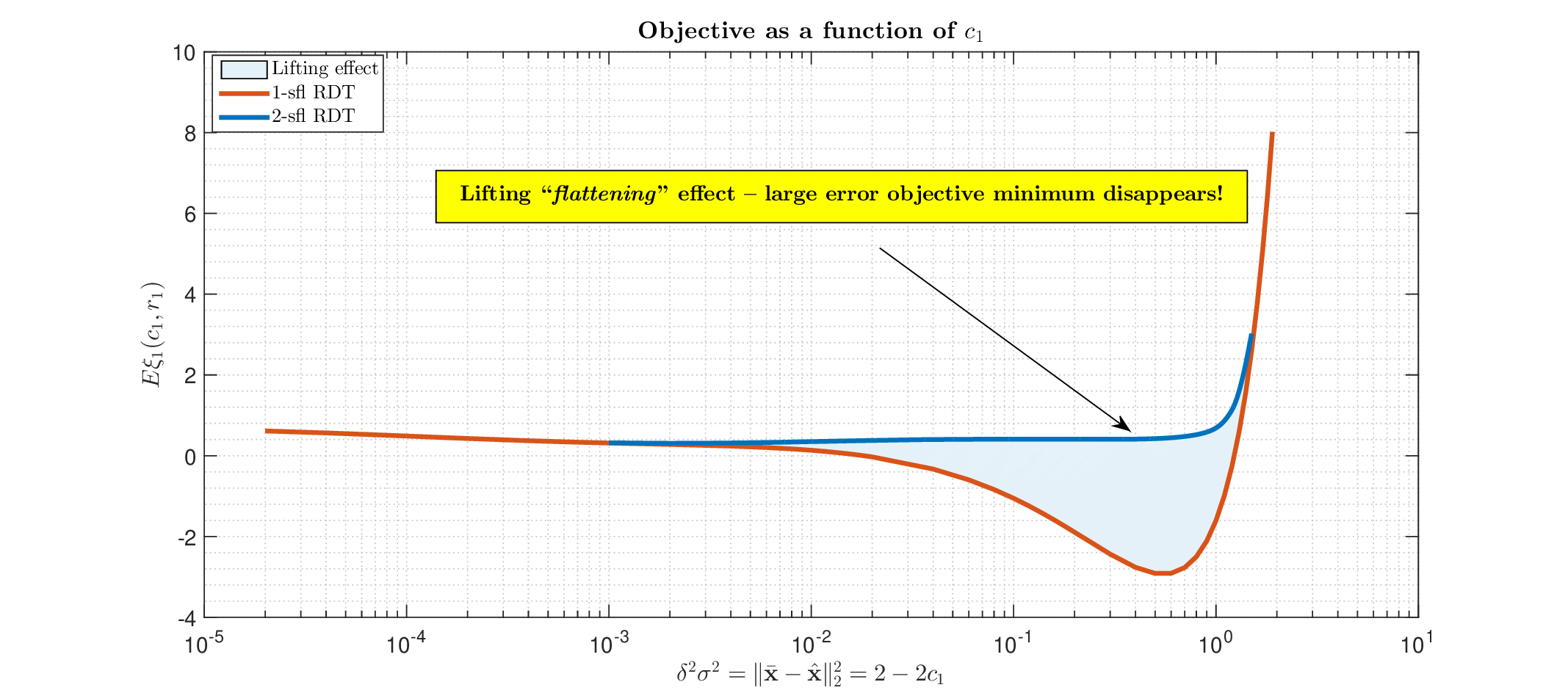}}
\caption{$\xi_1^{(2)}(c_1,r_1)$ as a function of $c_1$; $\alpha=0.59$, $\beta=0.4$, $\sigma=\sqrt{0.00025}$, and $r_1=1$; $\bar{\x}_i=\frac{{\cal N}(0,1)}{\sqrt{k}},1\leq i\leq k$.}
\label{fig:fig8}
\end{figure}
To get a bit of a feeling as to what is the source of such a dramatic change, we in Figure \ref{fig:fig9} show the 1-sfl binary and Gaussian results in parallel. As can be seen, in the binary case the 1-sfl unwanted large error dip happens when the curve is well on the rising resulting in the appearance of a strong local maximum. It is then difficult for 2-sfl to sufficiently lift the curve so that it can be flattened on the level of the local maximum. On the other hand, the Gaussian case does not face such an obstacle. While the dip still exists, it happens way before the curve starts rising and the difficult to overcome local maximum is nonexistent.
 \begin{figure}[h]
\centering
\centerline{\includegraphics[width=1\linewidth]{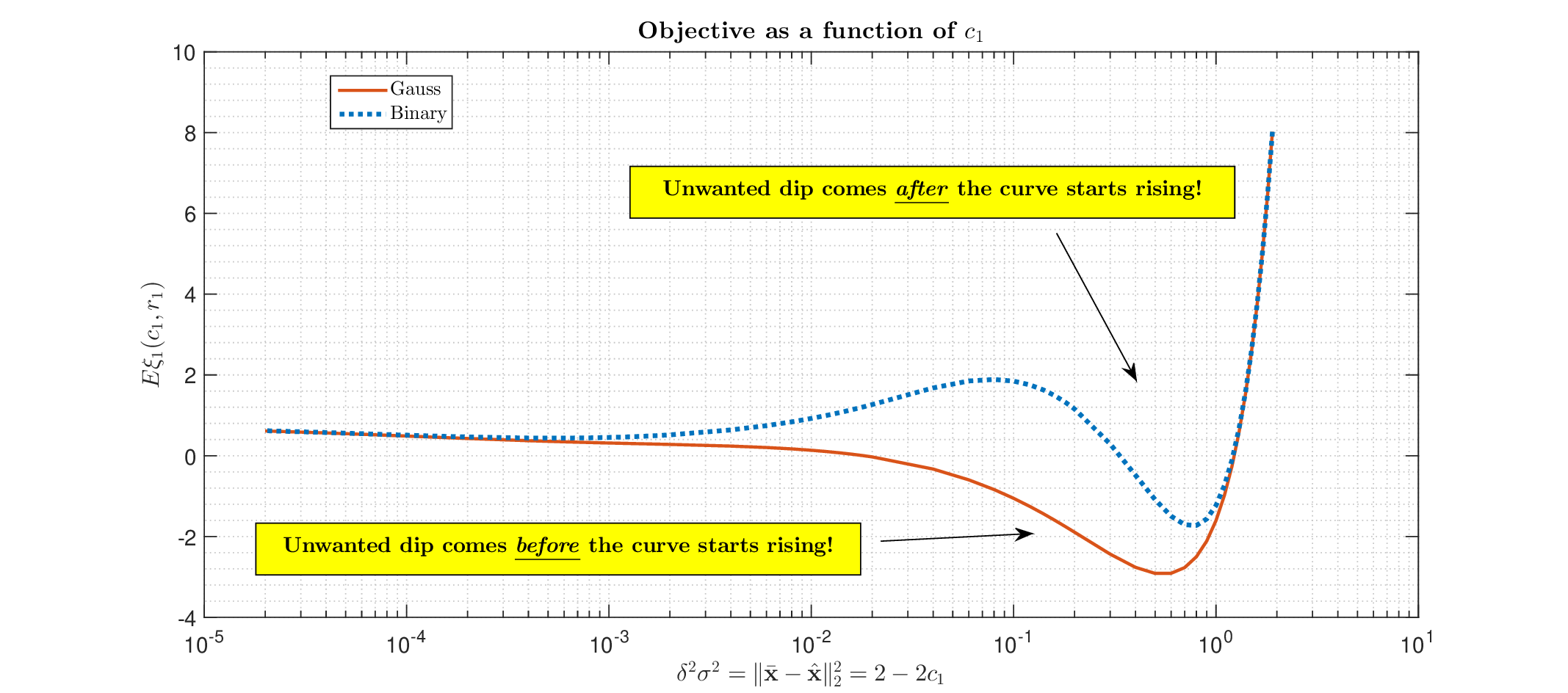}}
\caption{$\xi_1^{(2)}(c_1,r_1)$ as a function of $c_1$; $\alpha=0.59$, $\beta=0.4$, $\sigma=\sqrt{0.00025}$, and $r_1=1$; $\bar{\x}_i=\frac{{\cal N}(0,1)}{\sqrt{k}},1\leq i\leq k$.}
\label{fig:fig9}
\end{figure}

\section{Practical algorithmic aspects}
\label{sec:pracalg}

As discussed in earlier sections, the presented results have very strong consequences regarding the design of algorithmic techniques for solving $\ell_0$ ML decoding. Figure \ref{fig:fig7}, in particular, shows that depending on the values of the system parameters $\alpha$, $\beta$, and $\sigma$, the residual $\delta=\mE(\textbf{RMSE})$ can be large or small and achievable or unachievable by the descending $\ell_0$ norm preserving ($d\ell_0$) algorithms. Even though our prevalent interest in this paper is characterization of $d\ell_0$ limits of performance, we also conducted limited algorithmic tests utilizing a variant of the simplest possible $d\ell_0$. The utilized core variant is presented in Algorithm \ref{alg:algx1} (the core variant can then be repeated, i.e., restarted with different $\cA$ as many times as desired). It is basically the simplest possible bit flipping with a uniformly random bit selection (in a more adjusted optimization terminology, it is a steepest decent variant of the standard simplex methodology). Much more sophisticated \emph{simulated annealing} type of implementations can be employed though. As our prevalent interest here is the performance analysis, we provide a basic implementation just to illustrate that the presented concepts have rather significant algorithmic capabilities as well.

\begin{algorithm}[h]
\caption{Descending $\ell_0$ ($d\ell_0$)}

{\bf Input:} $\y \in \mR^m$, system matrix $A\in \mR^{m\times n}$,   $\cA$ - $k$ cardinality random subset of $\{1,2,\dots,n\}$,  maximum number of allowed iterations $i_{max}$, desired converging precision $\delta_{x,min}$.[\bl{$d\ell_0(\y,A,\cA,i_{max},\delta_{x,min})$}] \\
{\bf Output:} Estimated vector $\x_{(d\ell_0)} \in \mR^n$ [\bl{$\x_{(d\ell_0)}$}]

\label{alg:algx1}

\begin{algorithmic}[1]

\STATE Initialize the convergence gap and the iteration counter, $\delta_x\leftarrow 10^{10}$ and $i\leftarrow 1$

\STATE Set $\cB \leftarrow \{1,2,\dots,n\}\setminus\cA$

\STATE Set $\xi_1^{(0)} \leftarrow   \min_{\x\in\mS^k} \|\y-A_\cA\x\|_2$

\WHILE{$i\leq i_{max}$ and/or $\delta_x\geq\delta_{x,min}$}

\STATE Set $j$ as a random number from $\{1,2,\dots,k\}$

\FOR {$l=1:n-k$}

\STATE Set $\cC^{(l)} \leftarrow \lp \cA\setminus\cA_j\rp\cup \cB_l $

\STATE Set $a(l)\leftarrow   \min_{\x\in\mS^k} \|\y-A_{\cC^{(l)}}\x\|_2$

\ENDFOR

\STATE Set $l_{opt}\leftarrow \arg \min a $ and $\xi_{1}^{(i)}  \leftarrow \min a$

\IF {$\xi_{1}^{(i)}\leq  \xi_{1}^{(i-1)}$}

\STATE Set $\cA\leftarrow \cC^{(l_{opt})}  $

\ELSE

\STATE Set $\xi_{1}^{(i)} \leftarrow \xi_{1}^{(i-1)}$

\ENDIF

\IF {$i>2k$}

\STATE Set $\delta_x\leftarrow  |\xi_{1}^{(i)}- \xi_{1}^{(i-2k)} |$

\ENDIF

\STATE Update the iteration counter $i\leftarrow i+1$

\ENDWHILE

\STATE $\x_{(d\ell_0)}\leftarrow \x^{(i-1)}$.

\end{algorithmic}

\end{algorithm}

In Figure \ref{fig:figprac1}, we show the results obtained through Algorithm \ref{alg:algx1}. We selected $\alpha=0.2$ from a lower under-sampling regime and fairly small $n=200$. As can be seen, despite small dimensions, the simulated results are in an excellent agrement with the theoretical predictions. In fact, the agreement is beyond the best of the hopes as we are unaware of any algorithmic concept that achieves a prefect practical alignment with the theory for such small dimensions. It is also interesting to compare and see how far off are the results from the ones obtained in the related existing literature. As particularly relevant we consider the theoretical predictions obtained in the so-called \emph{ideal} ML context (where the locations of sparse components are a priori known). In such a scenario one has (see, e.g.,  \cite{Stojnicclupspreg20})
\begin{equation}\label{eq:algprac1}
  \delta_{ideal}= \mE(\textbf{RMSE}_{ideal})=\lim_{n\rightarrow \infty}\frac{\mE \|\bar{\x}-\hat{\x}_{ideal}\|_2}{\sigma}=\sqrt{\frac{\beta}{\alpha-\beta}}.
\end{equation}
Moreover, as discussed  in \cite{Stojnicclupspreg20} for CLuP algorithms in the ML context and  in, e.g., \cite{WuV12a,WuV12b,WuV11} for a  generic Bayesian context
\begin{equation}\label{eq:algprac1}
\lim_{\sigma\rightarrow 0} \delta= \lim_{\sigma\rightarrow 0}\lim_{n\rightarrow \infty}\mE (\textbf{RMSE})=\lim_{\sigma\rightarrow 0}\lim_{n\rightarrow \infty}\frac{\mE\|\bar{\x}-\hat{\x}\|_2}{\sigma}=\sqrt{\frac{\beta}{\alpha-\beta}}.
\end{equation}
As Figure shows, both the theoretical predictions and the simulated results are in an excellent agreement with these predictions as well. This basically means that $\sigma=\sqrt{0.00025}$ is sufficiently small that it almost identically emulates $\sigma\rightarrow 0$ in the binary case and is fairly close to it in the Gaussian case.
\begin{figure}[h]
\centering
\centerline{\includegraphics[width=1\linewidth]{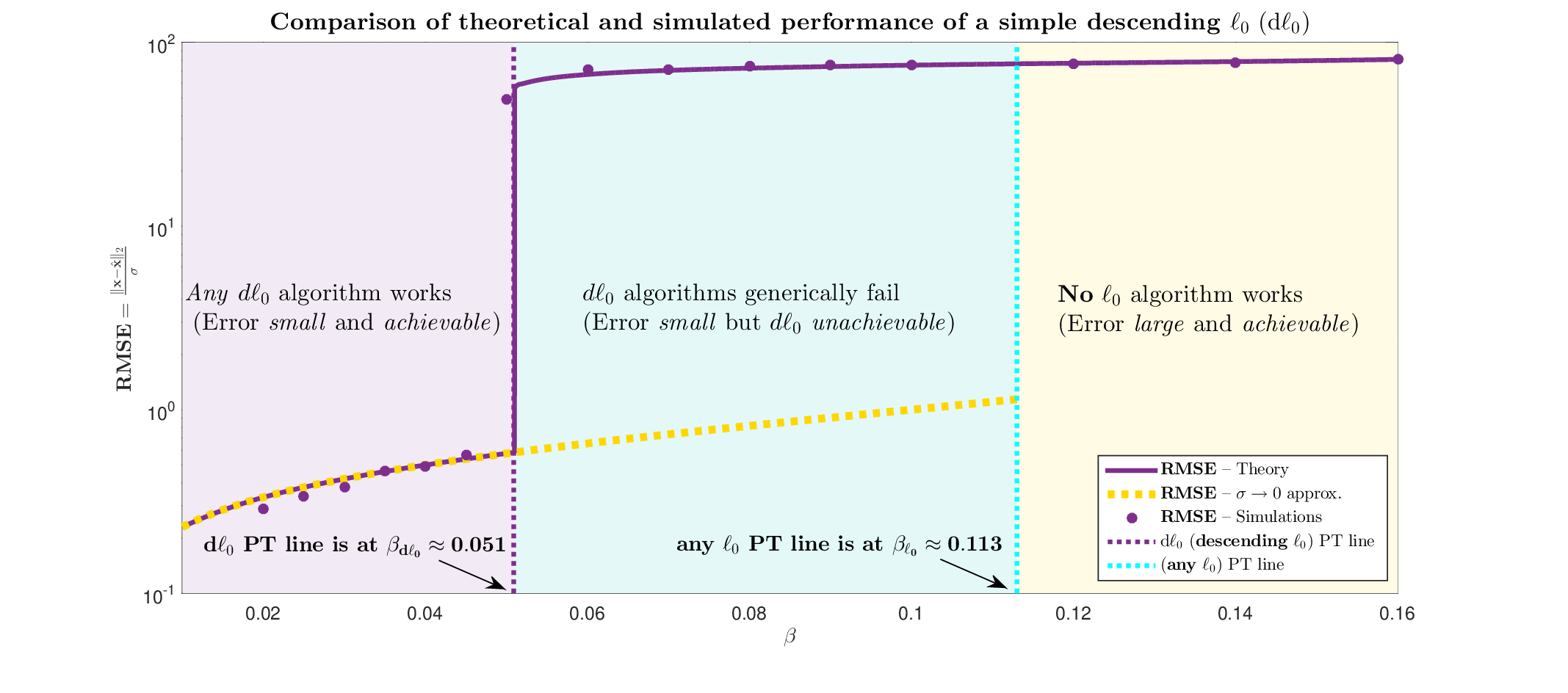}}
\caption{$\textbf{RMSE}$ as a function of sparsity ratio $\beta$; $\alpha=0.2$; $\sigma=\sqrt{0.00025}$, and $r_1=1$; $n=200$; Binary signal}
\label{fig:figprac1}
\end{figure}

In Figure \ref{fig:figprac2}, we show a comparison between the binary and the Gaussian case. We take $\alpha=0.59$ and consider even lower dimensions with  $n=100$. As can be seen, the phase transitions that simulated results give again fairly closely match the theoretically predicted ones. This again completely surpasses the best of the hopes as  we are completely unaware of any algorithm (beyond the exhaustive search) that can achieve such a performance.

We should also note that we plotted the median values of the errors as in the transition zones the error dramatically changes and the mean is not an adequate way to portray the algorithm's behavior (outside the transition zones the mean and median are pretty much identical).
The number of restarts did not exceed 100 in the transition zones of both Figure \ref{fig:figprac1} and Figure \ref{fig:figprac2} (outside transition zones hardly any restart was needed). As dimensions grow a better way to help the algorithm would be to use more sophisticated simulated annealing implementations where careful designs of the so-called temperature schedules can play a role of restarts and/or help with the finite size local minima effects more efficiently. As these considerations delve deeper into practical algorithmic aspects and our main focus here is theoretical performance analysis, we leave further discussions along these lines for separate papers.
\begin{figure}[h]
\centering
\centerline{\includegraphics[width=1\linewidth]{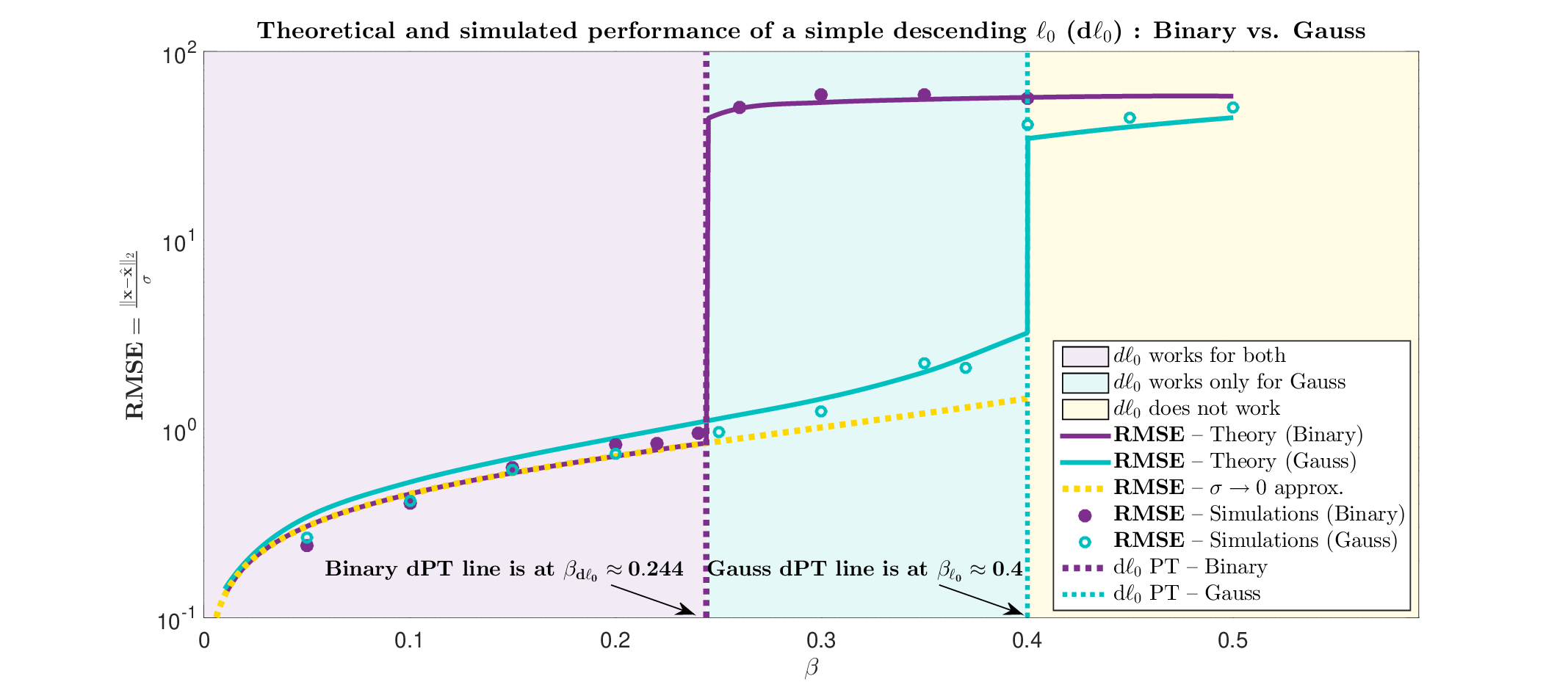}}
\caption{$\textbf{RMSE}$ as a function of sparsity ratio $\beta$; $\alpha=0.59$; $\sigma=\sqrt{0.00025}$, and $r_1=1$; $n=100$; Binary vs Gauss signals}
\label{fig:figprac2}
\end{figure}

\section{Conclusion}
\label{sec:conc}

We studied classical \emph{maximum-likelihood} (ML) recovery in noisy compressed sensing (CS) and sparse regression (SR). After incorporating sparsity (as a key CS and SR feature) in the recovery process, we particularly focused on $\ell_0$ variant of the ML decoding.  Utilizing \emph{Fully lifted random duality theory} (Fl RDT) from  \cite{Stojnicflrdt23} we developed a powerful generic program for studying $\ell_0$ ML. Assuming the so-called \emph{linear/proportional} regime in a statistical large dimensional context with a \emph{fixed deterministic a priori unknown} signal, two key analytical parameters were considered and precisely characterized: \textbf{\emph{(i)}} the \emph{objective value} of the resulting $\ell_0$ ML optimization; and \textbf{\emph{(ii)}} the residual \emph{root mean square error} $\textbf{RMSE}$.  When viewed as a function of the third key parameter, \emph{true-estimated signal overlap}, the objective is observed to exhibit critically important \emph{non-monotonic} behavior. A further coupling of such a behavior with the appearance of local optima is uncovered as well and directly connected to the existence of multiple $\textbf{RMSE}$ phase-transition (PT) phenomena. These phenomena are then precisely characterized with two PT curves: \textbf{\emph{(1)}} aPT curve, which separates the regions of systems dimensions where \emph{an} $\ell_0$ based algorithm succeeds or fails in achieving small (comparable to the noise) ML optimal $\textbf{RMSE}$; and \textbf{\emph{(2)}} dPT curve which does the same separation but for practically feasible \emph{descending} $\ell_0$ ($d\ell_0$) algorithms. In a more informal language, the first curve determines the ultimate theoretical limits of the $\ell_0$ ML, the second one determines the $d\ell_0$ achievable ones, and the region between them corresponds to the so-called \emph{computational gap} (C-gap) (where the $\ell_0$ ML decoding concept s useful but generic descending algorithms fail to solve it). Moreover, this effectively establishes ML complements to the Bayesian inference counterparts from, e.g., \cite{ReevesG12,ReevesG13,ReevPfis2016,BarbierKMMZ18,KMSSZ12,KMSSZ12a,WuV10a,WuV11,WuV12a,WuV12b}.

For Fl RDT to be practically relevant, performing a sizeable set of numerical evaluations is necessary. Uncovering a remarkable set of closed form  analytical relations among the key lifting parameters was of enormous help in handling numerical work. After completing all evaluations, we observed a remarkably fast convergence of the lifting mechanism with relative  corrections in the estimated quantities not exceeding $\sim 0.1\%$ already on the third level of lifting.

 Analytical results were supplemented with a sizeable set of the corresponding ones obtained through numerical experiments. We implemented a simple variant of $d\ell_0$ and demonstrated that its practical performance very accurately matches the theoretical predictions.   The agreement between the simulations and the theory is observed for fairly small dimensions of the order of 100.

A very generic character of the developed methodology ensures that many extensions and generalizations are possible. Among others, they include studying various forms of additional signal structuring as well as applications to other ML contexts. As the associated technical details are problem specific, we leave their further discussions for separate papers.

\begin{singlespace}
\bibliographystyle{plain}
\bibliography{nflgscompyxRefs}
\end{singlespace}

\end{document}